\documentclass[twoside]{IEEEtran}
\usepackage{graphicx}
\usepackage{subfigure}
\usepackage{algorithm} 
\usepackage[noend]{algpseudocode} 
\usepackage{amsmath, amsthm, amssymb}

\allowdisplaybreaks

\newtheorem{theorem}{Theorem}
\newtheorem*{theorem*}{Theorem}
\newtheorem{remark}{Remark}
\newtheorem{lemma}{Lemma}
\newtheorem*{lemma*}{Lemma}

\newtheorem{corollary}{Corollary}
\usepackage{wrapfig}
\DeclareMathOperator*{\argmax}{arg\,max}

\usepackage{mathtools}
\DeclarePairedDelimiter\ceil{\lceil}{\rceil}

\newcommand{\mcal}{\mathcal}
\newcommand{\mbb}{\mathbb}

\newcommand{\twopartdef}[4]
{
	\left\{
	\begin{array}{lll}
		#1 & \mbox{if } #2 \\
		#3 & \mbox{if } #4 
	\end{array}
	\right.
}
\usepackage{pifont}
\newcommand{\xmark}{\ding{55}}
\usepackage{xcolor}

\usepackage{algorithm}

\input{widebar_hack.tex}

\begin{document}

\title{Quantization of distributed data for learning} 

\author{Osama~A.~Hanna, Yahya~H.~Ezzeldin, Christina~Fragouli and Suhas~Diggavi
\thanks{
O.~A.~Hanna, C.~Fragouli and S.~Diggavi are with the Electrical and Computer Engineering Department at the University of California, Los Angeles, CA 90095 USA (e-mail: \{ohanna, christina.fragouli, suhasdiggavi\}@ucla.edu). Yahya~H.~Ezzeldin was with the Electrical and Computer Engineering Department at the University of California, Los Angeles, CA 90095 USA. He is now with the Electrical and Computer Engineering Department at the University of Southern California, CA 90089 USA (e-mail: yessa@usc.edu).

This work was partially supported by NSF grants \#2007714. \#1955632  and by UC-NL grant LFR-18-548554 and by Army Research Laboratory under Cooperative Agreement W911NF-17-2-0196.
}
}

\maketitle

\begin{abstract}
We consider machine learning applications that train a model by leveraging data distributed over a {trusted} network, where communication constraints can create a performance bottleneck. A number of recent approaches propose to overcome this bottleneck through compression of gradient updates. However, as models become larger, so does the size of the gradient updates. In this paper, 
{we propose an alternate approach to learn from distributed data}
that quantizes data instead of gradients, and can support learning over applications where the size of gradient updates is prohibitive. {Our approach leverages the dependency of the computed gradient on data samples, which lie in a much smaller space in order to perform the quantization in the smaller dimension data space. At the cost of an extra gradient computation, the gradient estimate can be refined by conveying the difference between the gradient at the quantized data point and the original gradient using a small number of bits. Lastly, in order to save communication, our approach adds a layer that  decides whether to transmit a quantized data sample or not based on its importance for learning.} We analyze the convergence of the proposed approach for smooth convex and non-convex objective functions and show that we can achieve order optimal convergence rates with communication that mostly depends on the data rather than the model (gradient) dimension. We use our proposed algorithm to train ResNet models on the CIFAR-10 and ImageNet datasets, and show that we can achieve an order of magnitude savings over gradient compression methods. {These communication savings come at the cost of increasing computation at the learning agent,  and thus our approach is beneficial in scenarios where communication load is the main problem. 
}

\end{abstract}

\section{Introduction}
Consider a machine learning application, where an agent wishes to train a model using data instances that are collected (and streamed potentially in real time) from a set of distributed terminal nodes. {We assume that the data is generated/collected at distributed nodes which do not have enough data to locally build good learning models. The scenario we consider is for a trusted network, so there are no privacy constraints on the data. Application scenarios for this include wireless sensor networks and IoT networks that collect data which is to be used for learning, e.g., an autonomous vehicle collecting data from by-the-road sensors{, and machine learning techniques used to overcome routing and tracking problems in wireless sensor networks \cite{kumar2019machine, nayak2021routing, warriach2013fault}}. Unlike work that aim for parallelizing the computation, we assume a setup where the learning is done centrally by an agent (server) from data that is naturally distributed. As a result, the key challenge that we address in this work,} is that the terminal nodes are connected to the server through a weak communication fabric; for instance, through wireless, bandwidth constrained links. The question we ask is: how can we compress the data so as to reduce the number of bits communicated, without affecting the learning performance of the agent.

\begin{figure}
	\centering
	\includegraphics[width=0.7\linewidth]{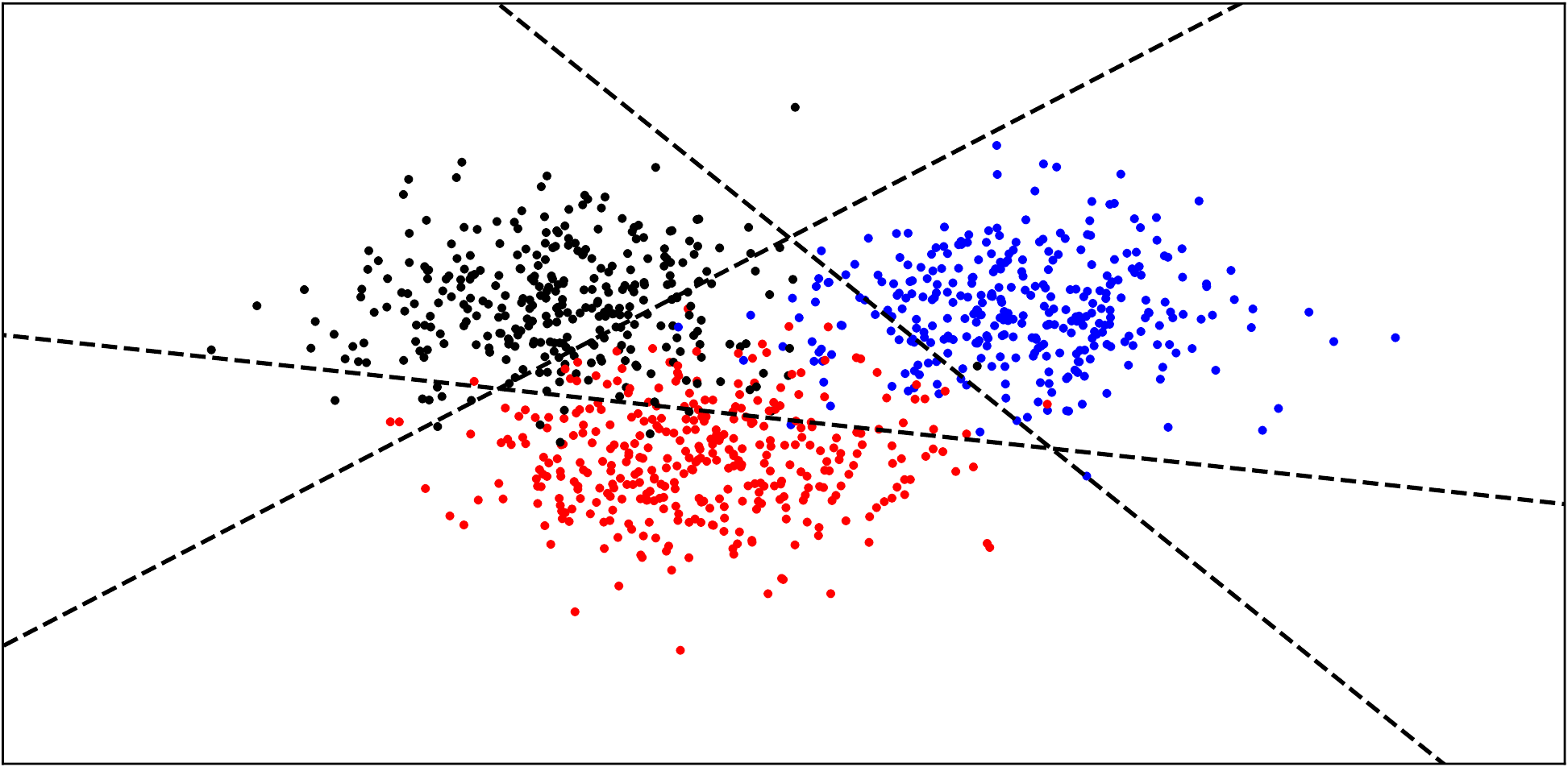}
	\caption{Building intuition for data sampling: training on all samples vs. only samples inside the triangle leads to classifiers with approximately the same performance.}\label{fig:red_intro}
	\vskip -0.2in
\end{figure}

Our main observation is that, although the nodes may continuously be observing data, not all the data is equally useful - perhaps a subset of the data 
is all we need. This can be thought of as a generalization  of the notion of support vectors in Support Vector Machines (SVM); effectively, we ask, which of the data instances  are most helpful for learning. 
As a simple example, we experimentally verified that training a neural network classifier on the data instances inside the triangle in Figure \ref{fig:red_intro} yields approximately the same validation accuracy as when we train the classifier on the full set of data instances. The challenge is that,
we would like to decide which are the instances to communicate in an online manner,  without having access to the whole dataset in advance.

How useful is an observed data instance depends on the corresponding stage of learning (eg., no data is useful if the model is already learned) and on the algorithm the agent uses for model training. In this paper we assume that the agent uses stochastic gradient descent, which is by far the most popular algorithm used for model training today. 

The fact that we restrict our attention to learning through stochastic gradient descent, naturally raises the following question. Leveraging the literature on federated learning, we could perhaps ask the terminal nodes to calculate gradients on their observed data, and then compress the gradient updates, following a technique such as in  \cite{seide20141, strom2015scalable, wen2017terngrad, alistarh2017qsgd,gandikota2019vqsgd, mayekar2020limits}. Indeed, this is a  reasonable solution (and we compare against it in our evaluation section~\ref{sec:numerics}), but has a major drawback: the size of models (and gradient updates) gets larger and larger - while the dimension of data does not. Indeed, in the pursuit of improved accuracy, neural networks  have evolved from thousands to millions to hundreds of millions
of parameters~\cite{touvron2020fixing,real2019regularized,devlin2018bert}.
However, the dimension (number of features) of data-instances themselves, whether images, natural signals, {or sensory data},  has not and is not expected to change much. 
Given that, unlike the federated learning setup, our data is not private, and our agent does not require computational support, it makes sense to leverage the lower dimensionality of data to achieve lower communication cost. {While there is work that aims to prune and reduce the size of machine learning models \cite{perryman2020pruned, perryman2016predicting}, our methods are still useful in those cases whenever the data dimension is smaller than the model dimension.}

In this paper we propose  a scheme that is based on quantizing data instances.
The proposed compression approach consists of two parts: (1) a selection scheme and (2) a quantization scheme. The selection scheme first decides which data instances are communicated depending on how important they are to learning; afterwards, these selected instances are quantized using our quantization scheme. 
The quantization scheme combines aspects of dataset quantization and gradient compensation {(if nodes have enough computational capabilities)},
and aims to enable the learning agent to perform gradient updates with low (or no) performance loss. {The optional gradient compensation step, in our scheme, enhances the learning performance by communicating only a few number of bits but comes at the cost of a gradient computation at a sensor node.}
We find that our scheme uses $\mathcal{O}(d\log_2(h))$ bits in the worst case to achieve the optimal convergence rate, where $d$ and $h$ are the dataset and model dimensions, respectively~\footnote{{Our theoretical analysis relies on a smoothness assumption of the gradient as a function of the data point; that is, if the true and quantized data points are close, then their gradients are also close.}}.
For large models ($h\gg d$), this is much  smaller than $\mathcal{O}(h)$, the 
information theoretic lower bound derived in~\cite{mayekar2020limits} on the number of bits needed to convey quantized gradients to achieve optimal convergence rate. This is because we take advantage of the data dependency of gradients, which is ignored by the oracle used for the lower bounds in \cite{mayekar2020limits}.

Our main contributions are:

$\bullet$
We exploit the fact that computed gradients are dependent on the datapoints to design $\mathtt{DaQuSGD}$, a  quantization approach  that only communicates $\mathcal{O}(d\log_2h)$ bits in the worst case per iteration, where $d$ is the dimension of a single datapoint and $h$ is the dimension of the model parameters.

$\bullet$
We additionally exploit the fact that only a small portion of  datapoints could be sufficient for a reliable model update in order to design an online sample selection approach that acts as a preliminary sampling step to $\mathtt{DaQuSGD}$. Our sample selection scheme provides savings in both communication and computation for the same performance.

$\bullet$ We analyze the convergence of our scheme for convex and non-convex objective functions, and prove that it achieves a similar convergence rate as unquantized SGD.

$\bullet$  We numerically show the benefits of using $\mathtt{DaQuSGD}$ for training ResNet models for CIFAR-10 and ImageNet datasets.
We achieve accuracy comparable to unquantized stochastic gradient descent, while using an order of magnitude less bits compared to gradient quantization schemes. 

\textbf{Related work:} Stochastic gradient descent has been very successful  in training neural networks \cite{seide20141,strom2015scalable,ghadimi2013stochastic}. A number of works proposed
quantized versions of stochastic gradient descent, see for example \cite{ seide20141, strom2015scalable,dean2012large, yu2014introduction, de2015taming, wen2017terngrad,  alistarh2017qsgd,gandikota2019vqsgd, mayekar2020limits}; these are based on gradient quantization 
and typically use  $\mathcal{O}(h)$ bits/iteration to achieve the same convergence rate as the unquantized gradient descent, which is proved to be a fundamental limit in~\cite{mayekar2020limits}\footnote{The reason why the lower bound in \cite{mayekar2020limits} can be broken is discussed in Section~\ref{sec:prelim}.} for memoryless gradient quantizers. 
Using memory with quantizers, sparsification methods were considered \cite{alistarh2018convergence, lin2017deep, stich2018sparsified, wu2018error, basu2019qsparse}, that improve communication efficiency at the cost of increasing the gradient iterations until convergence\footnote{Effective implementations of these also use  $\mathcal{O}(h)$ bits per gradient iteration.}. 

Data quantization was considered in~\cite{zhang2017zipml} for generalized linear models learning. However, in a generalized linear model, the gradient dimension is the same as the data dimension and the gradient in this case is only a scaled version of the data instance. In this paper, we consider data quantization for a general learning model. Our work is the first, as far as we know, that leverages dataset-based quantization to achieve the optimal convergence rate at a communication cost of $\mathcal{O}(d\log_2(h))$ that mostly depends on the dataset and not the model dimension. To further reduce communication cost, we propose a sample selection algorithm. 
There is a rich body of work on sample selection in centralized settings \cite{ng2003input, xu2019review, mirzasoleiman2020coresets, zhao2015stochastic}, however, these are either computationally expensive, cannot be applied to deep neural networks, require the knowledge of full dataset, or cannot be applied in a distributed setting. {For example, if every node has only one sample, applying the CRAIG algorithm proposed in~\cite{mirzasoleiman2020coresets} locally on each node's data would result in transmitting all the samples.} Schemes that are based on influence function \cite{jaeckel1972infinitesimal, hampel1974influence, cook1977detection} are also related. Influence functions approximate the effect of training samples on the model predictions during testing. This requires the computation of the gradient and Hessian of the model, which can be computationally expensive in deep networks. Moreover, influence functions are not well understood for deep models and non-convex functions \cite{koh2017understanding} and are shown to be fragile in deep learning \cite{basu2020influence}. We present and analyze a simple sample selection scheme that requires only the computation of the forward path of the model. {Importance sampling, such as~\cite{zhao2015stochastic}, studies weighted sampling methods with the goal of reducing the stochastic gradient variance. 
However, in such work, a point is going to be sampled in every iteration (but with a non-uniform optimized distribution). As a result, offering no communication cost saving. 
Furthermore, \cite{zhao2015stochastic} also requires knowledge of non-homogeneous upper bounds on the per sample gradient. Such knowledge is not currently feasible for deep models; this is also the reason why their approach~\cite{zhao2015stochastic} is only evaluated using an SVM model.} {Subsampling has also been utilized in~\cite{kuchnik2018efficient} in order to reduce the size of required dataset augmentation. The proposed schemes in~\cite{kuchnik2018efficient} require full knowledge of the dataset and targets to improve the model accuracy, in contrast to our goal of reducing the communication load.}
Our sample selection scheme can be applied for deep networks in distributed settings. This approach offers reduction in both communication and computation costs without sacrificing performance.

{There is a large body of research on active learning that aims to perform learning with few labeled instances, if unlabeled data are available but labels are expensive to obtain, by selecting few unlabeled instances to be labeled by an oracle~\cite{settles2009active, settles2008curious, tong2001active, monteleoni2006learning}. In contrast, all our data are labeled and we aim to decide what samples to not transmit to save communication.}

Our work also differs from traditional signal compression
\cite{gersho2012vector}:  our quantization does not aim to maintain the ability to reconstruct data, but instead the ability to perform  gradient updates with low (or no) performance loss.

{{\bf Notations:} In the rest of the paper, we use the following notation convention. We denote with $[a : b]$
the set of integers from $a$ to $b \geq a$; we also use the shorthand $[b]$ to denote integers from 1 to $b$. We use we use lowercase letters for scalars, uppercase letters for constants, and lowercase boldface letters for vectors. Calligraphic upper letters are used to denote sets of vectors or scalars.}

\section{Setup}\label{sec:model}
A learning agent  wants to use data from a space
$\mathcal{Z}\subseteq \mathbb{R}^d$ to generate a model in a space $\mathcal{W}\subseteq \mathbb{R}^{h}$.
The data samples are collected at distributed nodes that can communicate with the agent using at most $r$ bits per transmission. 
The overall dataset $\mathcal{D}=\{\mathbf{z}^{(1)},...,\mathbf{z}^{(N)}\}$ consists of $N$ samples from the space $\mathcal{Z}$ which is assumed to bounded, i.e., $\|\mathbf{z}\|_2\leq B${; we refer to such $\mcal{Z}$ as being $B$-bounded}.
We also assume that $\mcal{W}$ is bounded such that $\|\mathbf{w}-\mathbf{w}'\|_2^2 \leq D^2$, $\forall \; \mathbf{w}, \mathbf{w}' \in \mcal{W}${; we refer to such $\mcal{W}$ as being $D$-relatively bounded}.
The objective of the learning agent is to minimize the empirical risk of the output model. For a given model $\textbf{w}\in \mathcal{W}$, this risk is given by
\begin{equation}
L(\mathbf{w})=\frac{1}{N}\sum_{i=1}^N \ell(\mathbf{w},\mathbf{z}^{(i)}),
\end{equation}
where $\ell: \mathcal{W}\times \mathcal{Z}\to \mathbb{R}^+$ is a loss function. The loss function $\ell$ is  known at all nodes in the system. 

\noindent{\bf Learning algorithm.} We are interested in gradient-based learning algorithms.
Let $\nabla L(\mathbf{w})$ be the gradient of $L(\mathbf{w})$ and  $g_\mathbf{z}(\mathbf{w})=\nabla \ell(\mathbf{w},\mathbf{z})$ be
the gradient of the loss function (w.r.t. $\mathbf{w}$). It is well known that if 
$\mathbf{z}\in \mathbf{S}$ is sampled uniformly at random, $g_\mathbf{z}(\mathbf{w})$ is an unbiased stochastic gradient of $L(\mathbf{w})$, i.e., satisfies $\mathbb{E}[g_\mathbf{z}(\mathbf{w})]=\nabla L(\mathbf{w})$.
The gradient $g_\mathbf{z}(\mathbf{w})$ can be calculated at the node sampling $\mathbf{z}$ from the dataset.
Let $\widehat g_\mathbf{z}(\mathbf{w})$ be a general stochastic estimate of the gradient $\nabla L(\mathbf{w})$, calculated by the learning agent. This can be the true stochastic gradient $g_\mathbf{z}(\mathbf{w})$ estimated from a datapoint $\mathbf{z}$, or a low-precision mapping of $g_\mathbf{z}(\mathbf{w})$~\cite{alistarh2017qsgd,bernstein2018signsgd, suresh2017distributed,wen2017terngrad,mayekar2020limits}, or a function of both $g_\mathbf{z}(\mathbf{w})$ and $\mathbf{z}$.
Using this estimate, the model is updated at iteration $j+1$ as
\begin{align}\label{eq:update_rule_SGD}
	\mathbf{w}^{(j+1)}=\mathbf{w}^{(j)}-\eta \widehat g_{\mathbf{z}^{(j)}}\left(\mathbf{w}^{(j)}\right),
\end{align}
where $\eta$ is the learning rate. 

\noindent{{\bf System operational properties.}
{We assume that the learning agent (sometimes denoted as the server) and the distributed nodes make up a trusted network with the server being the strong capable entity at its center. As a result, we assume no privacy requirements and no computational constraints at the server. 
Additionally, the distributed nodes are connected to the server through a weak wireless communication fabric.
Thus, our focus, in this work, is on how to quantize data such that the learning algorithm converges with the same unquantized rate while reducing the communication overhead}\footnote{{Please note a general fact in wireless that uplink is much more costly than downlink as the later is a single broadcast transmission while uplink requires multiple transmissions depending on the number of nodes in the system.}}.}

\noindent {\bf Quantization.}
At the $j$-th iteration, a node $s_j$  has access to a datapoint $\mathbf{z}^{(j)}$. The learning agent generates
the stochastic gradient $\widehat g_{\mathbf{z}^{(j)}}\left(\mathbf{w}^{(j-1)}\right)$ by receiving information from the $s_j$-th node through an $r$-bit quantizer $Q_j$, where we assume that $Q_j$ has access to both the current datapoint $\mathbf{z}^{(j)}$, and the latest model $\mathbf{w}^{(j-1)}$. An $r$-bit quantizer consists of mappings $(Q^e_j, Q^d_j)$, with an encoder mapping $Q^e_j : \mcal{Z} \times \mcal{W} \to \{0,1\}^{\bar{r}}$ and a decoder mapping $Q_j^d : \{0,1\}^{\bar{r}} \to \mbb{R}^h$, where $\bar{r} \in \{0,r\}$. Note that in our setup we allow the encoder to transmit zero bits (by picking $\bar{r} = 0$) for some samples, which corresponds to not sending the sample.
The overall quantizer $Q_j = Q_j^d \circ Q_j^e$ captures the combined effect of the encoder and decoder.
Note that in the quantization schemes that directly quantize gradients,
the encoder mapping $Q_j^e$  first computes the stochastic gradient $g_{\mathbf{z}^{(j)}}(\mathbf{w}^{(j-1)})$ from the unquantized datapoint $\mathbf{z}^{(j)}$ and then maps $g_{\mathbf{z}^{(j)}}(\mathbf{w}^{(j-1)})$ to $r$ bits. In this paper, we explore an alternative approach where  $Q^e_j$  can make use of the datapoint, $\mathbf{z}^{(j)},$ directly. We highlight that we do not assume any extra information at the distributed nodes, i.e., they still only have access to the latest model and a full-precision datapoint.

\noindent{\bf Assumptions on the loss function.} In the following sections, we assume that the function $\ell(\mathbf{w},\mathbf{z})$ is $C_w$-smooth in $\mathbf{w}$ for all $\mathbf{z}\in \mathcal{Z}$, i.e., the gradient $g_\mathbf{z}(\mathbf{w}) = {\partial \ell(\mathbf{w},\mathbf{z})}/{\partial \mathbf{w}}$ is $C_w$-Lipschitz continuous in $\mathbf{w}$ for all $\mathbf{z}\in \mathcal{Z}$. In other words
\begin{equation}\label{assump1}
\| g_\mathbf{z}(\mathbf{w}) - g_\mathbf{z}(\mathbf{w}')\|_2 {\leq} C_w\|\mathbf{w}-\mathbf{w}'\|_2,\ \ \forall \mathbf{z} {\in} \mcal{Z}.
\end{equation}
Moreover, the gradient $g_\mathbf{z}(\mathbf{w})$ is assumed to be $C_z$-Lipschitz continuous in $\mathbf{z}$ for all $\mathbf{w}\in \mathcal{W}$, i.e.,
\begin{equation}\label{assump2}
\| g_\mathbf{z}(\mathbf{w}) - g_{\mathbf{z}'}(\mathbf{w})\|_2 {\leq} C_z\|\mathbf{z}-\mathbf{z}'\|_2,\ \ \forall \mathbf{w} {\in} \mcal{W}.
\end{equation}

While (\ref{assump1}) is a standard assumption \cite{alistarh2017qsgd,basu2019qsparse,mayekar2020limits}, the assumption in \eqref{assump2} is particular to our setup as we relate gradient updates to the datapoints. For this, we need that if a quantized point $\mathbf{z}_Q$ is close to its original point $\mathbf{z}$, then the gradients $g_\mathbf{z}(\mathbf{w}),g_{\mathbf{z}_Q}(\mathbf{w})$ are also close. We employ the assumption of Lipschitz continuity of the gradient in $\mathbf{z}$ in \eqref{assump2} to formalize this notion. Note that this assumption is implied if $\|\frac{\partial g_\mathbf{z}(\mathbf{w})}{\partial \mathbf{z}}\|_2\leq C_z$ \cite{bubeck2014convex}, which is the case in many loss functions (for instance, in the loss functions in the examples in Section~\ref{sec:prelim} below) since standard theoretical analysis in learning theory assume a bounded space for $\mathbf{z}$ and $\mathbf{w}$ \cite{shalev2014understanding}. We also highlight that in many cases (as in the loss function in the Example 2 in Section~\ref{sec:prelim} below, where $C_z=\mathcal{O}(h)$), $C_z$ might depend on $h$ or $d$. However, we will show that the effect of $C_z$ can be removed from the convergence rate with an increase in the number of bits that is $O(d\log_2(C_z))$. We observed in our numerical evaluation, on the CIFAR-10 dataset with ResNet-18, that our {\tt DaQuSGD} approach gives a ratio $\frac{\| g_\mathbf{z}(\mathbf{w}) - g_{\mathbf{z}'}(\mathbf{w})\|_2}{\|\mathbf{z}-\mathbf{z}'\|_2}$ that is upper bounded by $1$ when calculated using the sampled $\mathbf{z}$ and taking $\mathbf{z}'$ as its quantized version $\mathbf{z}_Q$.

\section{Preliminaries and motivation}\label{sec:prelim}

\subsection{Convergence of stochastic gradient descent.}
In our  setup,  the learning agent aims to learn a hypothesis of dimension $h$,
using stochastic estimates of the gradient of the empirical risk function $L$, where these  estimates are conveyed with a very low precision  representation $\{0,1\}^R$, that is, $R\ll h$. In this work, we are interested in providing convergence guarantees when learning from low precision estimates.

It is well known that for smooth convex risk functions, if stochastic gradient descent is performed with \emph{unbiased} stochastic gradients that have \emph{bounded variance} $\sigma^2$, then it convergences with  $\mathcal{O}\left({\frac{\sigma}{\sqrt{n}}}\right)$~\cite{bubeck2014convex}.
The exact theorem is provided in Appendix~\ref{app:sgdconv} for completeness. It was also shown in~\cite{agarwal2009information} that under the mentioned assumptions, $\mathcal{O}\left({\frac{\sigma}{\sqrt{n}}}\right)$ is a lower bound on the convergence rate of stochastic gradient descent. Hence, throughout the paper we refer to $\mathcal{O}(\frac{1}{\sqrt{n}})$ as the order-optimal convergence rate;
we take away that to achieve it, it suffices to construct a quantized stochastic gradient $\widehat g_\mathbf{z}(\mathbf{w})$ that satisfies the unbiasedness and bounded variance, where the variance bound should be $O(\sigma^2)$ (cannot grow with other system parameters such as $h,d$). We follow this well known technique to prove the order-optimal convergence of $\mathtt{DaQuSGD}$ in Section~\ref{sec:DQSGD}. 
Similar to the smooth convex risk functions case, for smooth non-convex risk functions, the stochastic gradient descent converges to a local optimal point provided that the stochastic gradients are unbiased and have bounded variance. The main theorem statement~\cite{ghadimi2013stochastic}, is reiterated in Appendix~\ref{app:sgdconv} for completeness.

\subsection{Direct gradient quantization is not always efficient.} The work in \cite{mayekar2020limits} shows that
if  gradients $g_\mathbf{z}(\mathbf{w})$ 
of the risk function $L(\mathbf{w})$ lie in a unit $\ell_2$ ball
and  the quantizer has no memory, then the minimum number of bits required for quantizing gradient to guarantee optimal convergence rate (same as unquantized gradient descent) is $\Omega(h)$ bits per iteration; this lower bound is derived using an oracle with access only to $g_\mathbf{z}(\mathbf{w})$  and not $\mathbf{z}$.
For memory-based quantizers, practical implementations typically require $\Omega(h)$ bits per iteration, even when sparsification methods~\cite{strom2015scalable,aji2017sparse,alistarh2018convergence} are used with quantization~\cite{basu2019qsparse}.

{This subsection is dedicated to examples where the $\Omega(h)$ lower bound does not hold. In these examples, we exploit the structure of the gradient to develop efficient custom quantization strategies that break the $\Omega(h)$ lower bound.} These examples do not suffer from the $\Omega(h)$ lower bound since the possible gradient values do not span the whole unit $\ell_2$ ball; the gradients live in a low dimensional space through their implicit dependence on $\mathbf{z}$. {Thus, this subsection motivates why in the following Section~\ref{sec:DQSGD}, we develop an approach that relies on data quantization instead of gradient quantization, while still aiming to achieve the optimal convergence rate.}

{\it Example 1:}
Consider the simple case where $d=1,h\gg d$, and the loss function is $\ell(\mathbf{w},\mathbf{z})=\frac{\mathbf{z}}{\sqrt{h}}\mathbf{1}^T\mathbf{w}$, where $\mathcal{W}=\{\mathbf{w}\in \mathbb{R}^h|\|\mathbf{w}\|_2\leq 1\},$ and $\mathcal{Z}=\{\mathbf{z}\in \mathbb{R}|\ |\mathbf{z}|\leq 1\}$. In this case, the best bound we can get is $\|g_\mathbf{z}(\mathbf{w})\|_2 \leq 1$. 
Consider a quantizer that only sends $1$ bit per data sample to the learning agent and operates as follows. At each time $j$, $Q^e_j, Q^d_j$ are chosen to be
\begin{align}\label{eq:quant}
&Q^e_j=Q^e(x)=\twopartdef
{1}      {\text{with probability } \frac{x+1}{2}}
{0}     {\text{with probability } \frac{1-x}{2}},\nonumber \\
 &Q^d_j(x)=Q^d(x)=2x-1.
\end{align}
The stochastic gradient is chosen to be $\frac{\mathbf{z}_Q^{(j)}}{\sqrt{h}}\mathbf{1}$, where $\mathbf{z}_Q^{(j)}$ is the output of decoder $Q^d_j$.
It is not difficult to see that the constructed gradient is an unbiased estimate of $\nabla L(\mathbf{w})$ with variance that is bounded by $1$. Hence, using this unbiased estimate, we can achieve the optimal convergence rate.
In this example, the gradient takes values only on the line segment $\left\{\frac{\mathbf{z}}{\sqrt{h}}\mathbf{1}\right| |\mathbf{z}|\leq 1\Big\}$; it does not span the whole unit $\ell_2$ ball.

\medskip

{\it Example 2}
Let us consider the loss function $\ell(\mathbf{w},\mathbf{z})=\frac{1}{\sqrt{h}}f(\mathbf{w}^T\mathbf{v}(\mathbf{z}))$, where
\[
\mathbf{v}(\mathbf{z})=[1,\mathbf{z},\mathbf{z}^2,...,\mathbf{z}^{h-1}],
\]
$\mathcal{W}=\{\mathbf{w}\in \mathbb{R}^h|\|\mathbf{w}\|_2\leq 1\}, \mathcal{Z}=\{\mathbf{z}\in \mathbb{R}|\ |\mathbf{z}|\leq 1\}$, 
and $f$ is a general $1$-smooth function. 
For instance, $f$ could be the logistic regression loss, i.e.,
\[
f(\mathbf{w}^T\textbf{v}(\mathbf{z})) = f_y(\mathbf{z}^T\textbf{v}(\mathbf{z})) = \log(1 + \exp(-y \mathbf{w}^T\textbf{v}(\mathbf{z}))).
\]
We here assume that $d=1$ for simplicity, and include the $d>1$ generalization at the end of this part. In this case, the best bound we can get on the norm of the stochastic gradient is $\|g_\mathbf{z}(\mathbf{w})\|_2 \leq 1$ as $\mathbf{z}$ is allowed to take the value $\mathbf{z}=1$.

We show next that the optimal convergence rate can be achieved with only $\log_2(h)$ bits per sample, instead of the expected $\Omega(h)$ bits.
Let $Q^e: [-1,1]\to \{0,1\}, Q^d:\{0,1\}\to [-1,1]$ be defined as in (\ref{eq:quant}).
The algorithm runs as follows. Each node transmits $1 + \ceil{\log_2(h)}$ bits representing the set of values $\{Q^e(f'(\mathbf{w}^T\mathbf{v}(\mathbf{z})))\}\cup \{Q^e(\mathbf{z}^{2^j})\}_{j=1}^{\ceil{\log_2(h)}}$, where $f'$ is the derivative of the function $f$. Let $b_j(i)$ be the $j$-th least significant bit in the binary representation of the integer $i$. 
The learning agent constructs $Q(\mathbf{z}^i)=\prod_{j=1}^{\ceil{\log_2(h)}} Q^d \circ Q^e(\mathbf{z}^{2^j})^{b_j(i)}\forall i\in [0:k]$ and the quantized stochastic gradient $\widehat{g}_\mathbf{z}(\mathbf{w}) = \frac{1}{\sqrt{h}}Q^d\circ Q^e(f'(\mathbf{w}^T\textbf{v}(\mathbf{z})))[1,Q(\mathbf{z}^1),...,Q(\mathbf{z}^k)]^T$. {In Appendix~\ref{app:ex}, we show that this algorithm achieves the optimal convergence rate and that in this example $\Omega(\log(h))$ is a lower bound on the communication cost for any SGD algorithm that achieves the optimal convergence rate.}

\textbf{Extension for $d{>}1$:} If $d{>}1$, the algorithm can be extended by sending a quantized version of $f'(\mathbf{w}^T\textbf{v}(\mathbf{z}))$ together with a quantized version of $[\mathbf{z}_1^{2^j},...,\mathbf{z}_d^{2^j}]$ for each $j\in [1{:}\ceil{\log_2(k)}]$ with $d$ bits using the $\mathtt{DataQ}$ with stochastic quantization that is described in Section \ref{sec:DQSGD} below. This would require $1+d\ceil{\log_2(k)}$ bits in total. Note that when $d>1$, each element of $\textbf{v}(\mathbf{z})$ corresponds to a monomial $\prod_{i=1}^d\mathbf{z}_i^{x_i}, \sum_{i=1}^dx_i\leq k$, hence, the size of the model is $h=O(k^d)$.

\section{Dataset-quantized stochastic gradient descent ($\texttt{DaQuSGD}$)}\label{sec:DQSGD}
This section presents $\mathtt{DaQuSGD}$, our proposed approach for dataset-based quantized stochastic gradient descent, which is summarized below and in the pseudocode in Algorithm~\ref{alg:DQSGD}.
 \begin{algorithm}
  \caption{$\mathtt{DaQuSGD}$}
  \label{alg:DQSGD}  
  \begin{algorithmic}
    \State Initialize $\mathbf{w}^{(0)}$, hyperparameter: $m$ number of levels in data quantizer; Suppose $\eta_j$ follows a learning rate schedule.
    \For{$j=1$ {\bf to} $n$}
    \State {\bf  Node selected in the $j$-th iteration:}
    \State Sample datapoint $\mathbf{z}^{(j)}$
    \State $\mathbf{z}_Q^{(j)}, \textbf{b}^{(j)}_\mathbf{z} \gets \mathtt{DataQ}\left(\mathbf{z}^{(j)},m\right)$;
    \State $\triangleright\ \textbf{b}^{(j)}_\mathbf{z}$ is the lossless binary compression of $\mathbf{z}_Q^{(j)}$
    \State $\Delta \gets g_{\mathbf{z}^{(j)}}\left(\mathbf{w}^{(j-1)}\right) - g_{\mathbf{z}^{(j)}_Q}\left(\mathbf{w}^{(j-1)}\right)$
    \State $ \textbf{b}^{(j)}_g \gets \mathtt{GradCorrQ}\left(\Delta, m\right)$;
    \State $\triangleright\ \textbf{b}^{(j)}_g\!$ is binary compression of estimate $\!\widehat\Delta\!$ of $\Delta$
    \State Send $\textbf{b}^{(j)}_\mathbf{z}$ and $\textbf{b}^{(j)}_g$ to  the learning agent
    \State ~
    \State {\bf Learning Agent:}
    \State Reconstruct $\mathbf{z}_Q^{(j)}$ and $\widehat\Delta$ from $\textbf{b}^{(j)}_\mathbf{z}$ and $\textbf{b}^{(j)}_g$
    \State $\widehat g \gets g_{\mathbf{z}^{(j)}_Q}\left(\mathbf{w}^{(j-1)}\right) + \widehat\Delta$
    \State $\mathbf{w}^{(j)} \gets \mathbf{w}^{(j-1)} - \eta_j\widehat g$
    \State Broadcast $\mathbf{w}^{(j)}$ to all  nodes
    \EndFor
    \State ~
    \State {\textbf{Define \emph{$\mathtt{DataQ}(\mathbf{z}, m)$}:}}
    \State {\texttt{ // The DataQ subroutine}}
    \State {\indent \underline{Hyperparameters}: Bound $B$ on the value of $\|\mathbf{z}\|_2$.}
    \State {\indent $\mathbf{z}^+ \gets \max(\mathbf{z},0)$, \quad $\mathbf{z}^- \gets \max(-\mathbf{z},0)$}
    \State {\indent $\textbf{a} \gets \left\lfloor\frac{(m-1)\mathbf{z}^+}{B}\right\rfloor$, \quad $\textbf{b} \gets \left\lfloor\frac{(m-1)\mathbf{z}^-}{B}\right\rfloor$; \quad $\textbf{a}, \textbf{b} \in \mbb{N}^d$}
    
    \State {\indent $\mathcal{S}=\left\{(\textbf{a},\textbf{b}) \in \mbb{N}^d \times \mbb{N}^d\ \right| \left.\ \|\textbf{a}\|_1 + \|\textbf{b}\|_1 \leq (m-1)^2\right\}$.}
    \State {\indent $\textbf{b}_\mathbf{z}^{(\textbf{a},\textbf{b})} \gets \texttt{Index}_\mathcal{S}\left((\textbf{a},\textbf{b})\right)$} {\texttt{//Index of $(a,b)$ in set $\mathcal{S}$ by enumeration}}
    \State {\indent {\bf Reconstruction.} $\mathbf{z}_Q \gets \left[\textbf{a}-\textbf{b}\right]\frac{B}{m-1}$; 
    }
    \State {\indent {\bf Return} $\mathbf{z}_Q$, $\textbf{b}_\mathbf{z}^{(\textbf{a},\textbf{b})}$}
    \State {\indent {\bf Comment:} The reconstruction step is also performed at the learning agent.}
  \end{algorithmic}
\end{algorithm}

The encoder $Q^e$ consists of two components: a datapoint quantization step, followed by a gradient correction step.
Given a datapoint $\mathbf{z}$, a node applies a data quantizer $\mathtt{DataQ}(\cdot,m)$ parameterized by an integer $m$, to create $\mathbf{z}_Q = \mathtt{DataQ}(\mathbf{z},m)$, the quantized version of $\mathbf{z}$. $\mathtt{DataQ}$ uniformly quantizes each feature of $|\mathbf{z}|$ into $m$ quantization levels as well as implicitly communicates the sign of each feature. 
At a high level, the quantizer $\mathtt{DataQ}$ can produce $\mathbf{z}_Q$ such that $\|g_{\mathbf{z}_Q}(\mathbf{w})-g_\mathbf{z}(\mathbf{w})\|_2$ is reasonably small. However, the quantized estimate $g_{\mathbf{z}_Q}(\mathbf{w})$ is no longer an unbiased estimate of $\nabla L(\mathbf{w})$. {In this case, we can only guarantee the convergence of $\nabla L(\mathbf{w}^{(n)})$ to a neighborhood of zero (i.e., to a neighborhood of a local optimal $\mathbf{w}$).}

{If we want to guarantee convergence to a local optimal point, then in next step a very low-precision amendment of $g_{\mathbf{z}_Q}(\mathbf{w})$ is sent} to ensure that the central agent has access to an unbiased estimate of $\nabla L(\mathbf{w})$ with bounded variance. Although the error $\Delta = g_{\mathbf{z}}(\mathbf{w}) - g_{\mathbf{z}_Q}(\mathbf{w})$ exists in a large dimensional space $\mbb{R}^h$, the fact that $\|g_{\mathbf{z}_Q}(\mathbf{w})-g_\mathbf{z}(\mathbf{w})\|_2$ is small  enables us to quantize $\Delta$ with a small number of bits (either 1 bit if the terminal nodes and the algorithm have shared randomness or $\lceil\log_2(h)\rceil + 1$ otherwise). Thus
the main quantization load is done in the data space of dimension $d\ll h$. 

{In the following subsections, we describe and analyze the two components of $\mathtt{DaQuSGD}$: Datapoint Quantization $\mathtt{DataQ}$ and Gradient Correction~$\mathtt{GradCorrQ}$.
For the different versions of $\mathtt{DaQuSGD}$ (with and without gradient correction), the order-wise convergence rates and their associated communication costs are summarized in Table~\ref{comp_table} and compared with the state-of the-art memoryless gradient quantization schemes.}
\begin{table*}[t]
\caption{Comparison of $\mathtt{DaQuSGD}$ (with and without gradient correction) with gradient compression schemes.}
\label{comp_table}
\begin{center}
{\begin{small}
\begin{sc}
\def\arraystretch{1.2}
\begin{tabular}{|l|c|c|c|c|}
\hline
\shortstack[l]{\textbf{Method}\\~} & \shortstack[l]{\textbf{Convergence}\\ \textbf{Guarantee}} & \shortstack[l]{\textbf{Communication} \\ \textbf{Cost/Iteration}} & {\shortstack[l]{\textbf{Learner} \\ \textbf{computes}\\\textbf{Gradient}}} & {\shortstack[l]{\textbf{Nodes} \\ \textbf{Computes}\\\textbf{Gradient}}}\\
\hline
$\mathtt{DaQuSGD}\  $($\mathtt{DataQ}$)    & $O(\frac{1}{\sqrt{n}} + \frac{1}{h})$ & $O(d\log_2(h))$ & {$\checkmark$} & {\xmark} \\
\hline
$\mathtt{DaQuSGD}\ $($\mathtt{DataQ}$ + $\mathtt{GradCorrQ}$) &  $O(\frac{1}{\sqrt{n}})$ & $O(d\log_2(h))$ & {$\checkmark$} & {$\checkmark$}\\
\hline
\shortstack[l]{$\text{Memoryless Gradient}$\\ $\text{Compression:}$\\QSGD~\cite{alistarh2017qsgd}, RATQ~\cite{mayekar2020limits}} &  \shortstack{~\\$O(\frac{1}{\sqrt{n}})$\\~} & \shortstack{~\\$O(h)$\\~} & {\xmark} & {$\checkmark$}\\
\hline
\end{tabular}
\end{sc}
\end{small}
}
\end{center}
\end{table*}

\subsection{Datapoint quantization \emph{[$\mathtt{DataQ}$]}}
    
$\mathtt{DataQ}(\cdot,m)$ 
is parametrized by  a tuning parameter $m$ capturing the number of levels used to quantize each feature (see also pseudocode in Algorithm~\ref{alg:DQSGD}).
Given a dataset sample $\mathbf{z}\in \mathbb{R}^d$, where $\|\mathbf{z}\|_2\leq B$,
we express it as $\mathbf{z}=\mathbf{z}^+-\mathbf{z}^-$, where $\mathbf{z}^+=\max(\mathbf{z},0)$ is a vector that captures the magnitude of the positive elements of $\mathbf{z}$ (negative values are replaced by zero) and similarly $\mathbf{z}^-=\max(-\mathbf{z},0)$ for the negative elements.
We create the non-negative concatenation vector  $\widetilde{\mathbf{z}} = [\mathbf{z}^+, \mathbf{z}^-]^T \in \mbb{R}_+^{2d}$, that inherits the following properties from $\mathbf{z}$: (i) $\widetilde{\mathbf{z}}$ has the same $\ell_2$ norm as $\mathbf{z}$; (ii) $\widetilde{\mathbf{z}}$ implicitly captures the sign of the $i$-th element of $\mathbf{z}$ by observing whether the $i$-th component of the subvector $\mathbf{z}^+$ is zero or not; (iii) If $\|\mathbf{z}\|_2 \leq B$, then $ 0 \leq \widetilde{\mathbf{z}} \leq B$ elementwise.
Then, for each coordinate $i$, we choose $m$ equally spaced quantization levels in the interval $[0,B]$, where the $i$-th quantization level is $q_i=\frac{iB}{m-1}$, $i\in [m\!-\!1]$.  
 Now, $\forall j \in [d]$, let 
\begin{align}\label{eq:a_b_definition}
	&\textbf{a}_j(\mathbf{z})  =\argmax_{i \in [m-1]}\{q_i|q_i\leq |\mathbf{z}^+_j|\},\nonumber \\
	&\textbf{b}_j(\mathbf{z})  =\argmax_{i \in [m-1]}\{q_i|q_i\leq |\mathbf{z}^-_j|\}.
\end{align}
The integer vectors $\textbf{a} = [\textbf{a}_1,..,\textbf{a}_d]^T$ and $\textbf{b}=[\textbf{b}_1,..,\textbf{b}_d]^T$ capture the indices of the levels that are just below the values in $\widetilde{\mathbf{z}}$; $\mathtt{DataQ}$ maps the values in $\widetilde{\mathbf{z}}$ to exactly the values indexed by $\textbf{a}$ and $\textbf{b}$. In particular, $\mathbf{z}^+$ is quantized to $\textbf{a} \frac{B}{m-1}$ and $\mathbf{z}^-$ is quantized to $\textbf{b} \frac{B}{m-1}$. Thus, As a result, we get that 
\begin{align*}
\mathbf{z}_Q = \mathtt{DataQ}(\mathbf{z},m) = (\textbf{a} - \textbf{b}) \frac{B}{m-1}.
\end{align*}

{\bf Communication Cost.} The quantized point $\mathbf{z}_Q$ generated by $\mathtt{DataQ}$ is uniquely represented by the integer vectors $\textbf{a}$ and $\textbf{b}$.
Thus, the communication cost equals the number of bits needed for lossless compression of $\textbf{a}$ and $\textbf{b}$.
Let us define the set $\mcal{S}$ to be 
\begin{align}\label{eq:set_S}
\mcal{S} = \Big\{(\textbf{v}^{(1)},\textbf{v}^{(2)}) \!\in\! \mbb{N}^d\!\times\!\mbb{N}^d\ \Big|\ \|\textbf{v}^{(1)}\|_1\!+\!\|\textbf{v}^{(2)}\|_1 \leq (m\!-\!1)^2 \Big\}.
\end{align}
We will argue that for any $(\textbf{a},\textbf{b})$ generated by $\mathtt{DataQ}$ as in~\eqref{eq:a_b_definition}, $(\textbf{a},\textbf{b}) \in \mcal{S}$ with probability 1.
As a result, we only need at most $\lceil \log_2(|\mcal{S}|)\rceil$ bits per sample to communicate $(\textbf{a},\textbf{b})$ to the central node.
We almost proved Theorem~\ref{th1} (the complete proof is provided in Appendix \ref{app:th1}) that provides 
an upper bound on the communication cost.
    

\begin{theorem}\label{th1}
	The proposed $\mathtt{DataQ}$ algorithm satisfies the following statements: (i) For the integer vectors $(\textbf{a},\textbf{b})$ uniquely defining $\mathbf{z}_Q$, we have that $ (\textbf{a},\textbf{b}) \in \mcal{S}$ with probability one; (ii) $\mathtt{DataQ}$ uses at most $2\log_2(m)+\min\{2d\log_2(e\frac{2d+m^2}{2d}),m^2\log_2(e\frac{2d+m^2}{m^2})\}$ bits per sample for communication; (iii) For the generated $\mathbf{z}_Q$, we have that $\|\mathbf{z}-{\mathbf{z}_Q}\|_\infty\leq \frac{B}{m-1}$ and $\|\mathbf{z}_Q\|_2\leq (1+\frac{\sqrt{d}}{m-1})B$.

\end{theorem}

\begin{remark} [Bounded Second Moment]
{\rm 
If the second moment satisfies $\mathbb{E}\left[\|\mathbf{z}\|^2_2\right]\leq B^2$, but $\|\mathbf{z}\|_2$ is not bounded almost surely,  we can send $\|\mathbf{z}\|_2$ with full precision and then use $\mathtt{DataQ}$ to quantize $\mathbf{z}/\|\mathbf{z}\|_2$, with $B=1$. This adds an overhead of sending (only) one scalar in full precision; moreover, we can now guarantee that $\mathbb{E}[\|\mathbf{z}_Q\|_2^2]\leq (1+\frac{\sqrt{d}}{m-1})^2$ instead of the universal bound stated on $\|\mathbf{z}_Q\|_2$ in Theorem~\ref{th1}. Additionally, we have that $\mathbb{E}[\|\mathbf{z}-{\mathbf{z}_Q}\|_\infty]\leq \frac{B}{m-1}$.
}
\end{remark}
\begin{remark}[Splitting $\mathbf{z}$]
{\rm 
  From Theorem~\ref{th1}, we have that the dependency of the number of bits on the dataset dimension $d$ is $\Omega(\log_2(d))$ when $m$ is small. However, without splitting $\mathbf{z}$ into $\mathbf{z}^+,\mathbf{z}^-$, we cannot represent the quantized values of $\mathbf{z}$, with a set of positive integers similar to $\mcal{S}$, without directly sending the signs in $\mathbf{z}$ which requires at least $d$ bits. Note that for small values of $d$, splitting might require larger number of bits than without splitting, however, we are interested in how the number of bits grow with $d$.
}
\end{remark}
\begin{remark}[Quantization of gradient using $\mathtt{DataQ}$]
 	{\rm
 	$\mathtt{DataQ}$ with stochastic quantization~\cite{alistarh2017qsgd,gray1993dithered}\footnote{Instead of mapping to a level below the feature, we map it either above or below with a probability depending on the distance.} if applied to the gradient $g_\mathbf{z}(\mathbf{w})$ with $m=\sqrt{h}+1$ would result in $\widehat g_\mathbf{z}(\mathbf{w})$ that satisfies 
 	$\mathbb{E}[\widehat g_\mathbf{z}(\mathbf{w})|g_\mathbf{z}(\mathbf{w})]=g_\mathbf{z}(\mathbf{w}),\|\widehat g_\mathbf{z}(\mathbf{w})-g_\mathbf{z}(\mathbf{w})\|_2\leq 2B$, hence, achieves the optimal convergence rate of $O(\frac{1}{\sqrt{n}})$.
  This uses at most $\log_2(2h)+2h\log_2(3e)=\mathcal{O}(h)$ bits, thus, achieving the communication lower bound in~\cite{mayekar2020limits}. Unlike the results in \cite{alistarh2017qsgd, suresh2017distributed} which provide a guarantee on communication cost in terms of expectation, this provides a uniform upper bound on the required number of bits.
 	}
 \end{remark}

\begin{remark}[Convergence of only $\mathtt{DataQ}$]
{\rm 
Note that, using~\eqref{assump2} and Theorem~\ref{th1}, we get
\begin{align}\label{eq:bounded_bias}
&\|g_{\mathbf{z}}(\mathbf{w})-g_{\mathbf{z}_Q}(\mathbf{w})\|_2 \stackrel{\eqref{assump2}}{\leq} C_z\|\mathbf{z}-\mathbf{z}_Q\|_2\nonumber \\
 &\quad \stackrel{\rm Th.~\ref{th1}}{\leq} C_z B\frac{\sqrt{d}}{m-1},\ \ \forall \mathbf{w} {\in} \mcal{W},
\end{align}
i.e., $g_{\mathbf{z}_Q}(\mathbf{w})$ is a biased estimate of $\nabla L(\mathbf{w})$, with a bounded bias given by~\eqref{eq:bounded_bias}. With this bias bound, we get the following convergence guarantee by appealing to the result in~\cite{agarwal2018cpsgd}.
 \begin{corollary}\label{cor:convergence_no_gradcorr}
    Let $\ell(\mathbf{w},\mathbf{z})$ be a function satisfying the assumptions in Section~\ref{sec:model}. By using the $\mathtt{DataQ}$ mechanism with parameter $m=h\sqrt{d}$, there exists a decaying learning rate $\eta$ such that
    \begin{equation}
    \mathbb{E}\left[\frac{1}{n}\sum_{i=1}^n \left\|\nabla L\left(\mathbf{w}^{(i)}\right)\right\|_2^2\right]\leq O\left(\frac{1}{\sqrt{n}} + \frac{1}{h}\right),    
    \end{equation}
    while using at most $O(d\log_2(h))$ bits per iteration.
\end{corollary}
The proof of Corollary~\ref{cor:convergence_no_gradcorr} is included in Appendix~\ref{app:proof_corollary_DataQ} for completeness.
The corollary implies that using $\mathtt{DataQ}$ guarantees convergence with rate $O(\frac{1}{\sqrt{n}})$ to a point with gradient $\nabla L(\mathbf{w}^{(n)})$ falling in a neighborhood of zero with radius $O(\frac{1}{h})$. Since $h$ is typically large, in our experiments in Section~\ref{sec:numerics}, we are able to train models to sufficiently good performance using only $\mathtt{DataQ}$.
}
\end{remark}
{To guarantee convergence to a local optimal point (with $\nabla L(\mathbf{w}) = 0$, instead of a neighborhood), in the following subsection, we introduce a gradient correction mechanism.}

\subsection{Gradient correction \emph{[$\mathtt{GradCorrQ}$]}}
\begin{algorithm}
	\caption{$\mathtt{GradCorrQ}(\Delta, m)$}
	\label{alg:GradCorrQ}  
	\begin{algorithmic}
		\State \underline{Hyperparameters} : $C_z$, Lipschitz constant of loss function in $\mathbf{z}$; $h$, dimension of hypothesis space; $d$, dimension of dataset.
		\State Initialize $\widehat \Delta = \mathbf{0}^h$
		\State Pick integer index $i^\star \in [h]$ uniformly at random.
		\State $t(i^\star) \gets$ binary representation of $i^\star$ using $\log_2(h)$ bits.
		\State $p \gets \frac{\Delta_{i^\star}}{2C_z B\frac{\sqrt{d}}{m-1}} + \frac{1}{2}$.
		\State $e_g \gets \texttt{Bernoulli}(p)$
		\State $\widehat\Delta_{i^\star} \gets (2e_g-1)C_z Bh\frac{\sqrt{d}}{m-1}$.
		\State {\bf Return} $\widehat{\Delta}$, $(t(i^\star), e_g)$.
		\State {\bf Comment:} Using $(t(i^\star), e_g)$, the central node can recreate $\widehat \Delta$.		
	\end{algorithmic}
\end{algorithm}
We here describe the gradient correction procedure $\mathtt{GradCorrQ}$ (see also the pseudocode in Algorithm~\ref{alg:GradCorrQ}) that augments the quantized gradient estimated from $\mathbf{z}_Q = \mathtt{DataQ}(\mathbf{z},m)$ to achieve $O(1/\sqrt{n})$ convergence.

Let $\Delta$ be the error in computing the gradient using $\mathbf{z}_Q$ defined as $\Delta=g_{\mathbf{z}}(\mathbf{w}){-}g_{\mathbf{z}_Q}(\mathbf{w})$. The node that sent $\mathbf{z}_Q$ to the learning agent can also communicate an estimate of $\Delta$, by quantizing it as follows.
The learning agent and the  node agree on a random number generator that generates $i^\star\in [1:h]$ uniformly at random. If there is no common randomness, then the  node can send $i^\star$ to the agent using $\lceil\log_2(h)\rceil$ bits. The  node also sends a single stochastic bit $e_g$, where $e_g$
\begin{align}
{\rm Pr}\Big\{e_g = 1\Big\} = \frac{C_z B\frac{\sqrt{d}}{m-1}-\Delta_{i^\star}}{2C_z B\frac{\sqrt{d}}{m-1}}.
\end{align}
This, in a way, uses random coordinate selection and stochastic quantization \cite{alistarh2017qsgd}. The agent constructs an estimate $\widehat \Delta$ and uses it to create the quantized stochastic gradient  as follows
\begin{align}
   &\widehat g_\mathbf{z}(\mathbf{w}) = g_{\mathbf{z}_Q}(\mathbf{w}) + \widehat{\Delta}, \mbox{ where }\nonumber \\
   &\widehat \Delta_i = \begin{cases}
(2 e_g - 1) C_z Bh\frac{\sqrt{d}}{m-1} & \text{if } i = i^\star, \\
0 & \text{otherwise }.
\end{cases}
\end{align}
Lemma~\ref{lem:grad_unbiased_estimate}, which is proved in Appendix \ref{app:lem2}, shows the unbiasedness and variance properties of the gradient estimate. 
From this, Theorem~\ref{thm:convergence_DQSGD}, which summarizes the convergence guarantees for $\mathtt{DaQuSGD}$ follows.
\begin{lemma}\label{lem:grad_unbiased_estimate}
	The quantized stochastic gradient $\widehat{g}_\mathbf{z}(\mathbf{w})$ is an unbiased estimate of  $ \nabla L(\mathbf{w})$ and satisfies $\|\widehat{g}_\mathbf{z}(\mathbf{w})-\nabla L(\mathbf{w})\|_2\leq C_zB\left(2 + (h+1)\frac{\sqrt{d}}{m-1}\right)$.
\end{lemma}
\begin{theorem}\label{thm:convergence_DQSGD}
$(1)$ Let $\mathcal{W}$ be convex and let $\ell(\mathbf{w},\mathbf{z})$ be a convex function  satisfying the loss function assumptions in Section~\ref{sec:model}. Let {$\mcal{Z}$ be $B$-bounded} and  $\sup_{\mathbf{w}\in \mathcal{W}}\|\mathbf{w}-\mathbf{w}^{(0)}\|^2_2 {=} \widetilde{D}^2$, where $\mathbf{w}^{(0)}{\in} \mathcal{W}$ is the initial model. Assume that the stochastic gradient descent uses the quantized  gradients $\widehat g_\mathbf{z}(\mathbf{w})$ obtained through {\rm $\mathtt{DaQuSGD}$}, with step size $\eta = \left(C_w+\frac{1}{\gamma}\right)^{-1}$, where $\gamma = \frac{\widetilde{D}}{\widehat{\sigma}}\sqrt{\frac{2}{n}}$. Then
	\begin{align}
	&\mathbb{E}\left[L\left(\frac{1}{n}\sum_{i=1}^n\mathbf{w}^{(i)}\right)\right]-L(\mathbf{w}^*)\leq \widetilde{D}\widehat{\sigma}\sqrt{\frac{2}{n}}+\frac{C_w \widetilde{D}^2}{n},\nonumber \\
	&\mbox{ with } \widehat\sigma = C_zB\left(2 + (h+1)\sqrt{d}/m-1\right). \nonumber
 	\end{align}
	
	$(2)$ Let $\ell(\mathbf{w},\mathbf{z})$ be a function (possibly non-convex) satisfying the loss function assumptions in Section~\ref{sec:model}, and $\|\nabla L(\mathbf{w})\|_2\!\leq\! \widetilde{D},\ \forall \mathbf{w}\!\in\!\!\mathcal{W}$\footnote{Note that the assumptions in Section~\ref{sec:model} imply that $\|\nabla L(\mathbf{w})\|_2$ is bounded almost surely.}. Let $L(\mathbf{w}^{(0)})-L(\mathbf{w}^*)\!=\! D_0$, where $\mathbf{w}^{(0)}{\in} \mathcal{W}$ is the initial model and $\mathbf{w}^*$ is the optimal model. Then, {\rm $\mathtt{DaQuSGD}$} with step size $\eta = \min\{C_w^{-1},\gamma\}$, where $\gamma = \frac{1}{\widehat\sigma}\sqrt{\frac{2D_0}{nC_w}}$, satisfies	
  \begin{equation}
	\frac{1}{n}\!\sum_{t=1}^n \mathbb{E}[\|\nabla L(\mathbf{w}^{(t)})\|_2^2] \leq 2\widehat\sigma \sqrt{\frac{2C_wD_0}{n}}+\frac{2D_0C_w}{n},\!
	\end{equation}
	where $\widehat\sigma$ is as before, and expectation is over the random selection of points, and randomness in {\rm $\mathtt{DaQuSGD}$}.
\end{theorem}
Using the definition of $\widehat\sigma$ in Theorem~\ref{thm:convergence_DQSGD} we can find the condition required for optimal convergence using $\mathtt{DaQuSGD}$, which is stated by the following corollary.

\begin{corollary}
For $m=h\sqrt{d}$, {\rm $\mathtt{DaQuSGD}$} achieves the optimal convergence rate, using at most $1 + \log_2(h)+2\log_2(h\sqrt{d})+2d\log_2(e(1+h^2/2))$ bits per iteration.
\end{corollary}

\section{Selecting samples to transmit}\label{sec:redund}
In this part, we aim to further reduce the communication cost of $\mathtt{DaQuSGD}$ using sample selection. The high level incentive is the following: \emph{If a data sample contributes minimally to the model learning, we do not transmit it}. This can intuitively be viewed as a generalization of support vectors in SVM; we want to only send samples that are necessary for designing the classifier, and we also want to decide
which samples are necessary in an \emph{online}
way, \emph{without the knowledge of the full dataset}. Note that, samples that are transmitted are still quantized using $\mathtt{DaQuSGD}$. 

Our sample selection method assumes terminal nodes have enough resources to perform the forward pass of the model that we want to train at the central node. 
First, for a new data sample in the $i$-th iteration, we apply the last version of the model available at the terminal node $\mathbf{w}^{(i-1)}$ on the new data sample $\mathbf{z}$ to compute the loss $\ell(\mathbf{w}^{(i-1)},\textbf{z})$. Using this loss, we quantize and transmit the point to the central node only if the loss exceeds a threshold, i.e., if $\ell(\mathbf{w}^{(i-1)},\mathbf{z}){>} \ell_{th}^{(i)}$, where $\ell_{th}^{(i)}$ is the threshold in iteration $i$.

 The intuition behind this loss thresholding approach is that for classification tasks, samples close to the classification boundary (points in P2 in Fig.~\ref{fig:linear_boosting_example}) will typically have higher loss (important for the classifier design), while the samples that are well away from the boundary will exhibit smaller loses and thus have a minimal effects on the model (P1 and P3 in Fig.~\ref{fig:linear_boosting_example}).
\begin{figure}
	\begin{center}
	\includegraphics[width=0.6\linewidth]{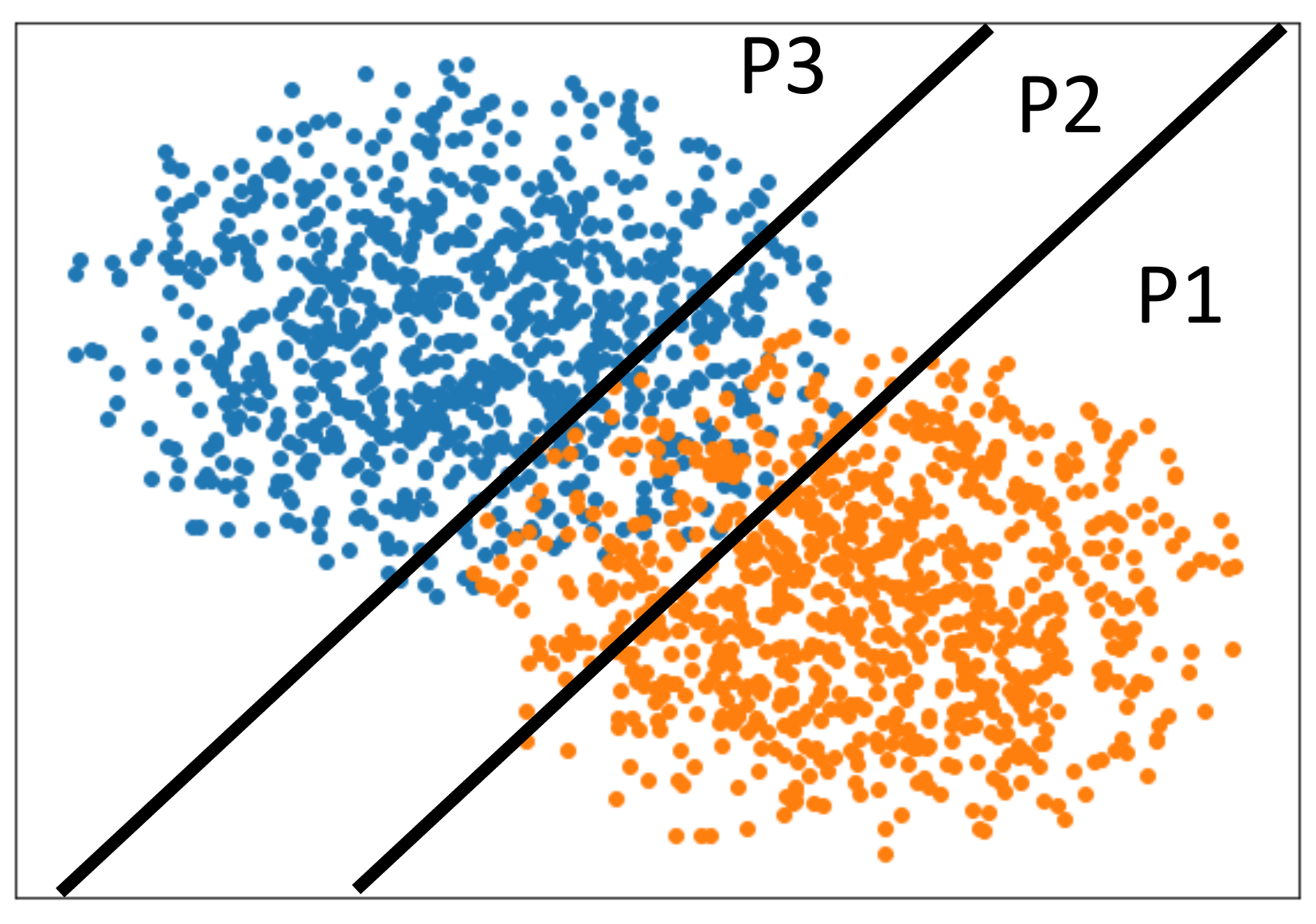}
	\end{center}
	\caption{Illustrative example for sample selection.}\label{fig:linear_boosting_example}
	\vskip -0.2in
\end{figure}

Formally, in the $i$-th iteration of the learning algorithm, by applying the selection based on thresholding, the algorithm minimizes the risk function $\hat{L}(\mathbf{w})=\frac{1}{N}\sum_{i=1}^N \max \{\ell_{th}^{(i)},\ell(\mathbf{w},\mathbf{z}^{(i)})\}$. Hence, by decreasing the threshold $\ell_{th}^{(i)}$ as the algorithm progresses, the algorithm approaches the optimal point for a convex loss function (or a local optimal for non-convex loss function). Ideally, we want the loss of points that are not communicated to not be detrimental to the convergence rate of the algorithm. This cannot be guaranteed in general. The following theorem (proved in Appendix~\ref{app:sampsel}) provides a sufficient condition on designing the thresholds $\{\ell_{th}^{(i)}\}_{i=1}^n$ such that our thresholding approach can converge with rate $\mcal{O}(\frac{1}{\sqrt{n}})$, same as the vanilla SGD.
\begin{theorem}\label{thm:thresholding}
	Let $\mathcal{W}$ be convex, $L(\mathbf{w})$ be convex, $C_w$-smooth, and $\sup_{\mathbf{w}\in \mathcal{W}}\|\mathbf{w}-\mathbf{w}^{(0)}\|^2_2 {=} \widetilde{D}^2$, where $\mathbf{w}^{(0)}{\in} \mathcal{W}$ is the initial hypothesis. Let  $\mathbf{w}^*=\arg\min_{\mathbf{w}\in \mathcal{W}}L(\mathbf{w})$ and assume that $L(\mathbf{w}^\star) \geq 0$. Assume that SGD is performed with stochastic gradients $\widehat g(\mathbf{w})$ that satisfy (i) $\mathbb{E}[\widehat{g}(\mathbf{w})]=\nabla L(\mathbf{w})$ (unbiasedness), and (ii) $\mathbb{E}[\|\widehat{g}(\mathbf{w})-\nabla L(\mathbf{w})\|^2_2]\leq \widetilde{B}^2$ (bounded variance). Additionally, assume that the loss thresholds satisfy that $\sum_{i=1}^n\sqrt{\ell_{th}^{(i)}}\leq \sqrt{n}$, then at iteration $n$, if the step size is $\eta = \min\{C_w^{-1},\left(\sqrt{n}\widetilde{B}\right)^{-1}\}$, we have
	\begin{align}
		\mathbb{E}&\left[L\left(\frac{1}{n}\sum_{i=1}^n\mathbf{w}^{(i)}\right)\right]-L(\mathbf{w}^*)\leq \nonumber \\
		&\qquad\frac{\widetilde{B}(\widetilde{D}/2+1)+\widetilde{D}\sqrt{2C_w}}{\sqrt{n}}+\frac{3C_w\widetilde{D}\sqrt{2C_w}}{n\widetilde{B}}.
	\end{align}
\end{theorem}
Hence, this approach reduces communication cost, and computation cost (since backpropagation is only applied on a subset of the samples) without sacrificing the order of the convergence rate. In fact, we show numerically, in Section~\ref{sec:numerics}, that thresholding can provide an improvement in terms of how fast do we converge to a good performance model.
\begin{remark}[Sample Selection + $\mathtt{DaQuSGD}$]
{\rm 
Theorem~\ref{thm:thresholding} assumes only the unbiasedness and bounded variance properties of the gradient estimates. Thus, we can apply the sample selection through thresholding on top of our $\mathtt{DaQuSGD}$ quantization approach described earlier in Section~\ref{sec:DQSGD} without a penalty in the order of convergence, since it generates gradient estimates with these properties at the central node.
}
\end{remark}
\section{Experimental Evaluation}\label{sec:numerics}
We evaluate the practical gain $\mathtt{DaQuSGD}$ achieves in terms of communicated bits for training ResNet models on the CIFAR-10 and ImageNet datasets. 

\textbf{ImageNet Experimental Setup.} For our ImageNet experiments, we train a ResNet-50~\cite{he2016deep} ($h=25,557,032$) on the ImageNet \cite{deng2009imagenet} dataset ($d=150,528$). We use a learning rate schedule with a base learning rate of 1e-1 with a piece-wise decay of 0.25 introduced every 30 epochs and batch size of 128. {We consider a scenario with a single learning server and 10,000 distributed nodes, each holding a unique set of 128 ImageNet images locally. The collection of all disjoint image sets from all distributed nodes make up the ImageNet training set.}
The training was done using SGD with momentum of 0.9.

\textbf{CIFAR-10 Experimental Setup.} For CIFAR-10 ($d=3,072$), we train a ResNet-18 ($h=11,173,962$).  We use a learning rate schedule consisting of a base learning rate of 1e-3 with a piece-wise decay of 0.1 introduced at epoch 80 and batch size of 128. {We assume a single central server that learns the classifier using data from 390 distributed nodes, where each distributed node stores 128 CIFAR images locally (totaling 49,920 total data points)}. Similar to the ImageNet models, the ResNet-18 networks were trained using SGD with momentum of 0.9.

{\bf {DaQuSGD} Implementation.} We make some implementation adjustments to the $\mathtt{DaQuSGD}$ algorithm theoretically analyzed for convergence in Section~\ref{sec:DQSGD}.
Our  implementation  differs as follows.
\begin{figure*}[t!]
  \centering
\subfigure[top-1 accuracy for~$\mathtt{DaQuSGD}$ and~$\mathtt{QTopK\mbox{-}SGD}$.]{\includegraphics[width=0.38\linewidth]{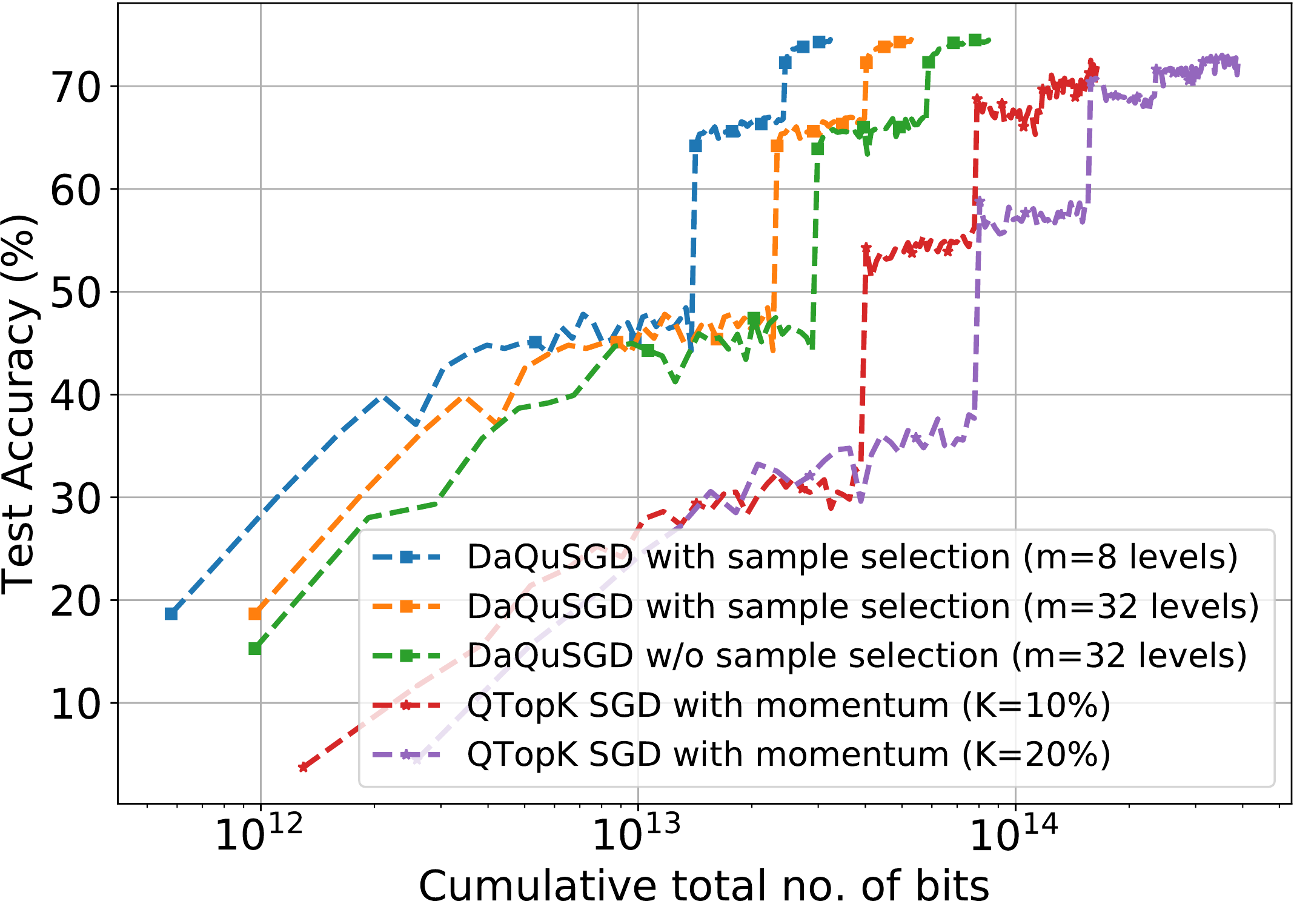}}\label{fig:ImageNet}
\subfigure[Number of bits required to achieve $72\%$ top-1 accuracy.]{\hspace{0.4in}\includegraphics[width=0.38\linewidth]{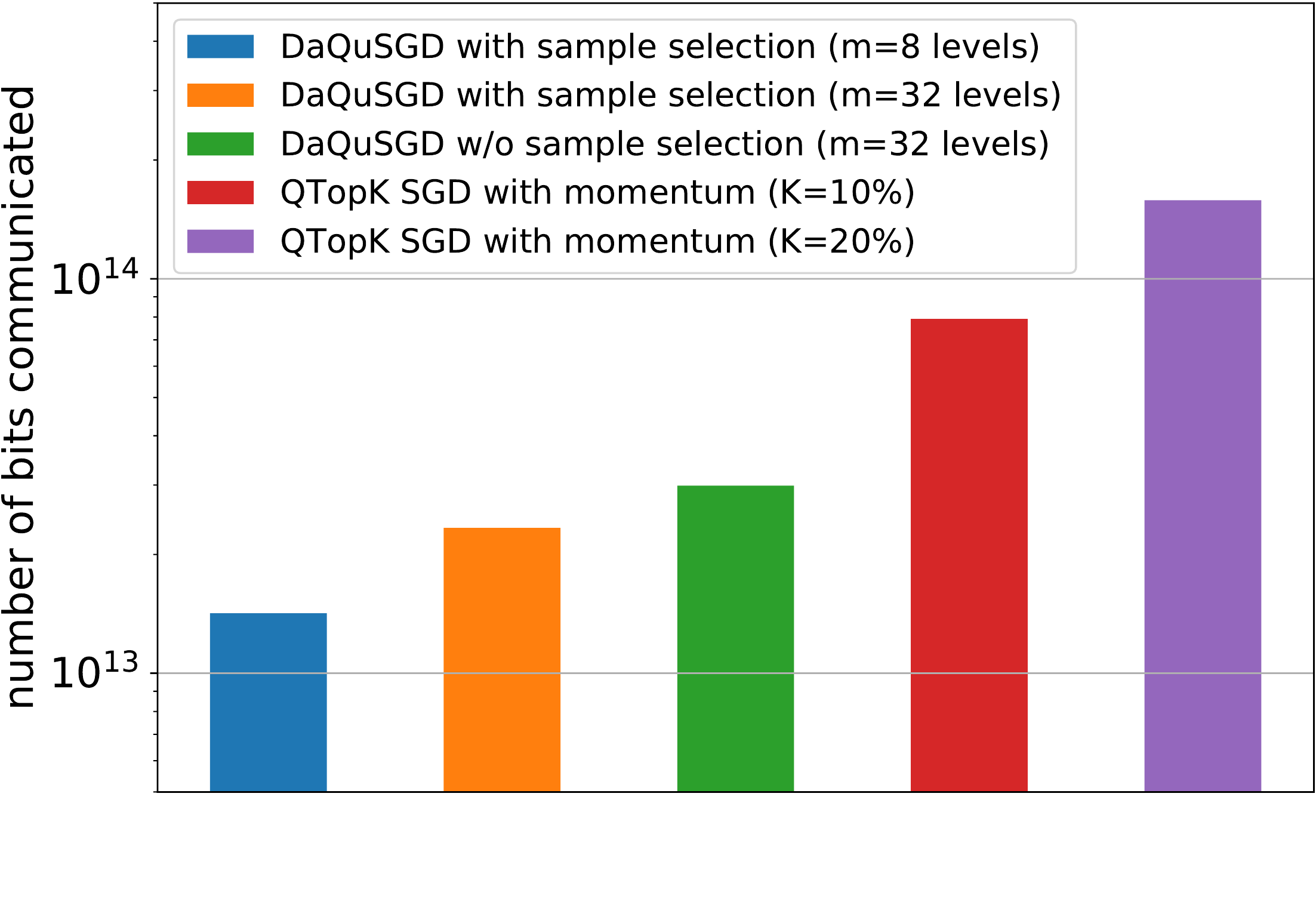}\hspace{0.4in}}\label{fig:ImageNet_comm_to_acc}
\caption{Performance of $\mathtt{DaQuSGD}$ in comparison with $\mathtt{QTopK\mbox{-}SGD}$ when training ResNet-50 on ImageNet.}\label{fig:cumulative_both_imagenet}
  \vskip -0.1in
\end{figure*}
\begin{figure*}[t!]
\centering
 \subfigure[top-1 accuracy for~$\mathtt{DaQuSGD}$ and~$\mathtt{QTopK\mbox{-}SGD}$.]{\includegraphics[width=0.38\linewidth]{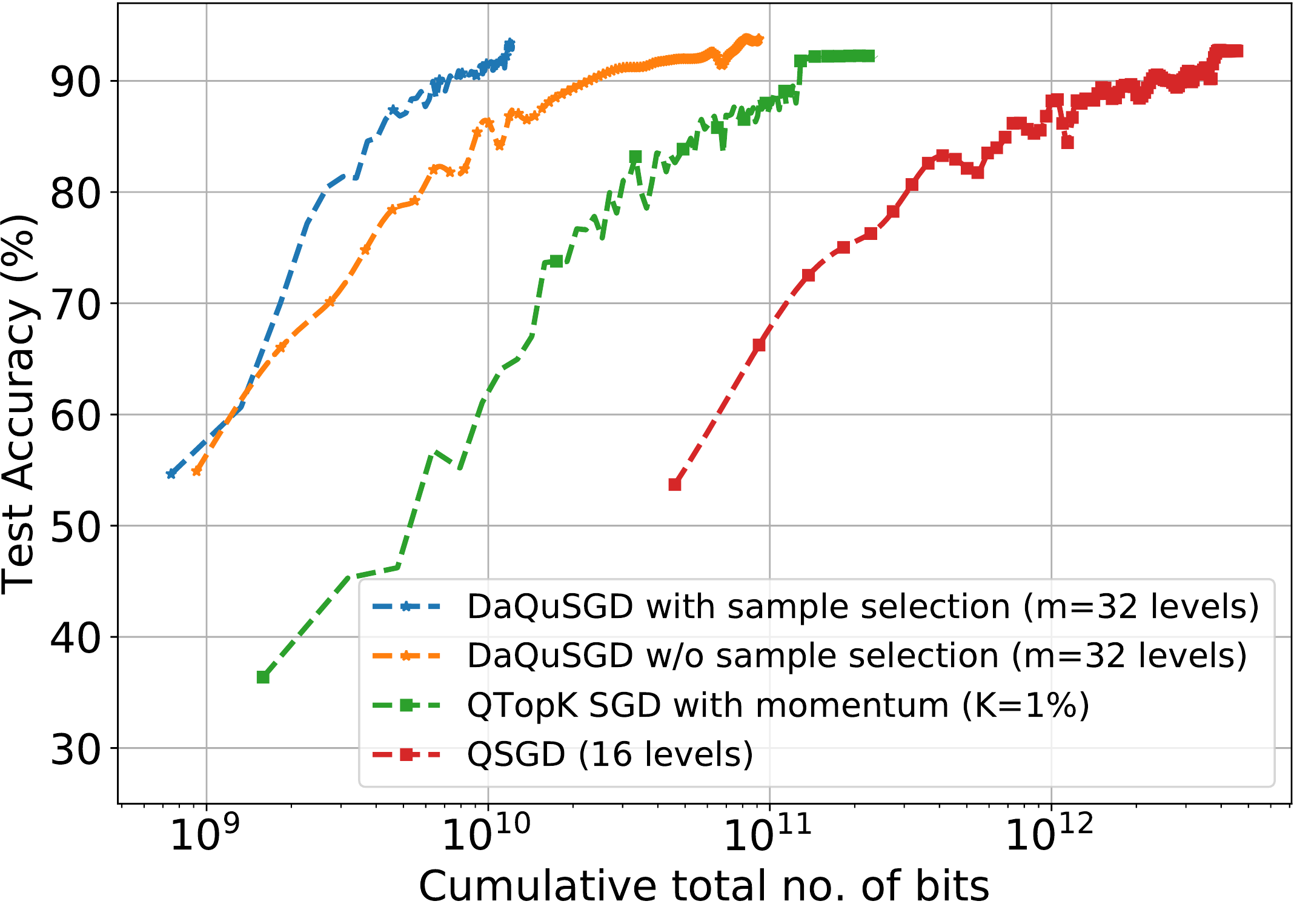}}\label{fig:cifar}
  \subfigure[Number of bits required to achieve $92.3\%$ top-1 accuracy.]{\hspace{0.4in}\includegraphics[width=0.38\linewidth]{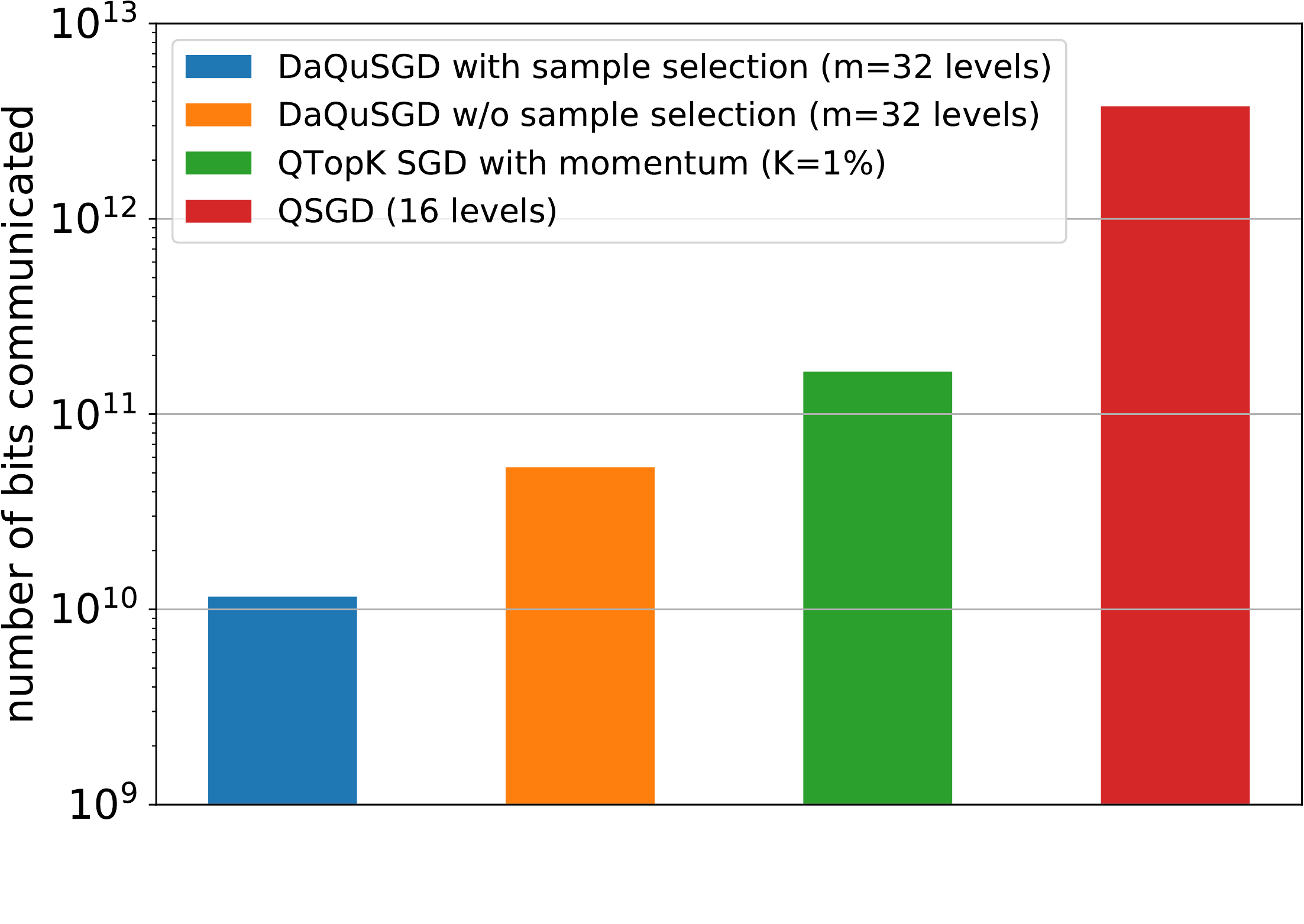}\hspace{0.4in}}\label{fig:cifar_comm_to_acc}
  \caption{Performance of $\mathtt{DaQuSGD}$ vs. $\mathtt{QTopK\mbox{-}SGD}$ and $\mathtt{QSGD}$ when training ResNet-18 on CIFAR-10.}\label{fig:cumulative_both_cifar}
  \vskip -0.1in
\end{figure*}
First, instead of scaling each data vector of the batch by its own $\ell_2$-norm before applying $\mathtt{DataQ}$, we scale all of them  using $\max_{i \in [1:128]}\|\mathbf{z}^{(i)}\|_\infty$, the maximum feature value observed in the batch. This helps retain more values close to their original unquantized values and allows to get away with communicating a single full-precision scaling factor instead of a factor for each datapoint. 

{As observed through Corollary~\ref{cor:convergence_no_gradcorr} and Theorem~\ref{thm:convergence_DQSGD}: when $\mathtt{GradCorrQ}$ is not used, the convergence guarantee on $\nabla L(\mathbf{w})$ is worsened from being to zero to instead being to a ball of radius $O(\frac{1}{h})$ around zero. This is substantially small for $h > 10^7$ as in ResNet models.}
We also saw minimal performance gains in our experiments when $\mathtt{GradCorr}$ was used. 
Thus in our experiments, we opted to not use $\mathtt{GradCorrQ}$, and instead use only $\mathtt{DataQ}$. 
As a result, terminal nodes at most only need to calculate the loss function and do not need to run backpropagation. 



When sample selection is applied, the threshold at the beginning of each epoch is set to be $0.2$ of the average loss of the transmitted samples of the previous epoch.

{\bf Comparison to direct gradient quantization.}
On ImageNet, we compare the performance of $\mathtt{DaQuSGD}$ with sample selection to the state-of-the art gradient sparsification and quantization approach $\mathtt{QTopK\mbox{-}SGD}$ with memory \cite{basu2019qsparse} and momentum \cite{singh2020squarm}. $\mathtt{QTopK\mbox{-}SGD}$ quantizes the top $K\%$ gradient values per model block using 16 bits each and performs error compensation by keeping memory of the error between true and compressed gradients. 

The ResNet-50 model was trained on the ImageNet dataset for $90$ epochs using our $\mathtt{DaQuSGD}$ approach and $150$ epochs using the sparse $\mathtt{QTopK\mbox{-}SGD}$ approach. 

Figure~\ref{fig:cumulative_both_imagenet}(a) shows the growth in cumulative communication budget during training versus the test accuracy (recorded at every epoch). {For gradient quantization schemes, the number of bits communicated is the number of bits used to represent the quantized gradient that is generated at the local node using each of the considered schemes. Note that, this value is independent of the batch size in case of schemes that quantize gradients. For our proposed data quantization, each node communicates a number of bits dependent on its local batch, which is equal to the number of bits used to represent each point multiplied by the number of points that were allowed through by the sample selection module (or the full batch size if sample selection is not used).} 

{The final accuracies for $\mathtt{DaQuSGD}$ with sample selection ($m{=}8$), $\mathtt{DaQuSGD}$ without sample selection ($m{=}8$), and $\mathtt{QTopK\mbox{-}SGD}$ ($K=10\%$) are $73.63\%,73.48\%,72.16\%$, respectively}. 
We see that our proposed approach offers a saving of up to a factor of $7$ over $\mathtt{QTopK\mbox{-}SGD}$, for the same accuracy. 

Figure~\ref{fig:cumulative_both_imagenet}(b) illustrates that $\mathtt{QTopK\mbox{-}SGD}$ requires a higher communication budget to converge to the same accuracy for the $\mathtt{DaQuSGD}$ with sample selection.
This higher communication budget is manifested through the larger number of training epochs ($33\%$ more than what $\mathtt{DaQuSGD}$ utilized) in order to achieve the same accuracy. Larger number of epochs  translates to  both a longer training time as well as additional computational cost. That is, $\mathtt{DaQuSGD}$ with sample selection does not require memory, and saves complexity in two ways: it allows to converge to the same top-1 accuracy  using fewer epochs and it only requires to compute the gradient for a subset of samples at every epoch. 

For training ResNet-18 over the CIFAR-10 dataset, our proposed model $\mathtt{DaQuSGD}$ and $\mathtt{QSGD}$ was trained for 100 epochs; $\mathtt{QTopK\mbox{-}SGD}$ was trained for 150 epochs.
On CIFAR-10, we observed even more significant gains for our $\mathtt{DaQuSGD}$ approach over the $\mathtt{QTopK\mbox{-}SGD}$ and the $\mathtt{QSGD}$~\cite{alistarh2017qsgd} as seen in Fig.~\ref{fig:cumulative_both_cifar}(a) and Fig.~\ref{fig:cumulative_both_cifar}(b). In particular, our approach provided a factor of $14$ reduction in the communication cost over $\mathtt{QTopK\mbox{-}SGD}$ and a factor of $300$ over $\mathtt{QSGD}$, to achieve the same top-1 accuracy. {The final accuracies for $\mathtt{DaQuSGD}$ with sample selection, $\mathtt{DaQuSGD}$ without sample selection, $\mathtt{QTopK\mbox{-}SGD}$ and $\mathtt{QSGD}$ are $92.57\%,93.15\%,92.31\%,92.39\%$, respectively}.




\appendices
\section{SGD standard convergence theorems}\label{app:sgdconv}
The central theorem used to prove convergence of stochastic gradient descent for smooth and convex risk functions is~\cite[Theorem 6.3]{bubeck2014convex}, which modified to our notation is expressed as: 
\begin{theorem}[{\cite[Theorem 6.3]{bubeck2014convex}}]\label{thm:sgd_convergence}
	Let $\mathcal{W}$ be convex, $L(\mathbf{w})$ be convex, $C_w$-smooth, and $\sup_{\mathbf{w}\in \mathcal{W}}\|\mathbf{w}-\mathbf{w}^{(0)}\|^2_2 {=} \widetilde{D}^2$, where $\mathbf{w}^{(0)}{\in} \mathcal{W}$ is the initial hypothesis. Let  $\mathbf{w}^*=\arg\min_{\mathbf{w}\in \mathcal{W}}L(\mathbf{w})$. Assume that stochastic gradient descent is performed with stochastic gradients $\widehat g(\mathbf{w})$ that satisfy (i) $\mathbb{E}[\widehat{g}(\mathbf{w})]=\nabla L(\mathbf{w})$ (unbiasedness), and (ii) $\mathbb{E}[\|\widehat{g}(\mathbf{w})-\nabla L(\mathbf{w})\|^2_2]\leq \widetilde{B}^2$ (bounded variance), then at iteration $n$, if the step size is $\eta = \left(C_w+\frac{1}{\gamma}\right)^{-1}$ and $\gamma = \frac{\widetilde{D}}{\widetilde{B}}\sqrt{\frac{2}{n}}$, 
	\begin{equation}
	\mathbb{E}\left[L\left(\frac{1}{n}\sum_{i=1}^n\mathbf{w}^{(i)}\right)\right]-L(\mathbf{w}^*)\leq \widetilde{D}\widetilde{B}\sqrt{\frac{2}{n}}+\frac{C_w\widetilde{D}^2}{n}.
	\end{equation}
\end{theorem}

For smooth non-convex risk functions, the stochastic gradient descent converges to a local optimal point. The main result we use in our analysis is from \cite{ghadimi2013stochastic} is stated below.
\begin{theorem}[\cite{ghadimi2013stochastic}]\label{thm:non-convex_convergence}
Let $L(\mathbf{w})$ be $C_w$-smooth, and $\|\nabla L(\mathbf{w})\|_2\leq \widetilde{D},\ \forall \mathbf{w}\in \mathcal{W}$. Let $L(\mathbf{w}^{(0)})-L(\mathbf{w}^*)= D_0$, where $\mathbf{w}^{(0)}{\in} \mathcal{W}$ is the initial hypothesis and $\mathbf{w}^*$ is the optimal hypothesis. Assume that stochastic gradient descent is performed with stochastic gradients $\widehat g(\mathbf{w})$ that satisfy (i) $\mathbb{E}[\widehat g(\mathbf{w})]=\nabla L(\mathbf{w})$ (unbiasedness), and (ii) $\mathbb{E}[\|\widehat g(\mathbf{w})-\nabla L(\mathbf{w})\|^2_2]\leq \widetilde{B}^2$ (bounded variance), then at iteration $n$  with step size $\eta = \min\{C_w^{-1},\gamma\}$ and $\gamma = \frac{1}{\widetilde{B}}\sqrt{\frac{2D_0}{nC_w}}$, we have that
	\begin{equation}
	\frac{1}{n} \sum_{t=1}^n \mathbb{E}[\|\nabla L(\mathbf{w}^{(t)})\|_2^2]\leq 2\widetilde{B}\sqrt{\frac{2C_wD_0}{n}}+\frac{2D_0C_w}{n},
	\end{equation}
	where the expectation is taken over the distribution of the stochastic gradient $\widehat g(\mathbf{w})$.
\end{theorem}

\section{Extended discussion for Example 2 in Section~\ref{sec:prelim}} \label{app:ex}
{In this appendix, we show that the algorithm proposed for Example 2 in Section~\ref{sec:prelim} achieves the optimal convergence rate with $\mathcal{O}(\log(h))$ bits and that $\Omega(\log(h))$ is a lower bound on the number of bits used by any SGD algorithm that achieves the optimal convergence rate for this example.}

{To show that the algorithm achieves the optimal convergence rate, it suffices to prove that it outputs and unbiased estimate of the gradient with a variance bounded by a constant.} It is easy to see that any number in the set $\{1,...,k\}$ can be expressed as the sum of power of two of numbers from the set $\{1,...,\lceil\log_2(k)\rceil\}$, where each element appears in the sum at most once. With this, we now show that $\widehat{g}_\mathbf{z}(\mathbf{w})$ is an unbiased stochastic gradient. By definition, we have that
\begin{align}
	\mathbb{E}[\widehat g_\mathbf{z}(\mathbf{w})_i|\mathbf{z}]
	&=\frac{1}{\sqrt{h}}\mathbb{E}[Q^d\circ Q^e(f'(\mathbf{w}^T\textbf{v}(\mathbf{z})))\nonumber \\
	&\qquad\qquad\prod_{j=1}^{\ceil{\log_2(k)}}(Q^d \circ Q^e(\mathbf{z}^{2^j}))^{b_j(i)}|\mathbf{z}]\nonumber \\
	&\stackrel{[i]}=\frac{1}{\sqrt{h}}\mathbb{E}[Q^d\circ Q^e(f'(\mathbf{w}^T\textbf{v}(\mathbf{z})))|\mathbf{z}] \nonumber \\
	&\qquad\qquad\prod_{j=1}^{\ceil{\log_2(k)}}\mathbb{E}[(Q^d\circ Q^e(\mathbf{z}^{2^j}))^{b_j(i)}|\mathbf{z}]\nonumber \\
	&\stackrel{[ii]}=\frac{1}{\sqrt{h}}\mathbb{E}[Q^d\circ Q^e(f'(\mathbf{w}^T\textbf{v}(\mathbf{z})))|\mathbf{z}] \nonumber \\
	&\qquad\qquad\prod_{j=1}^{\ceil{\log_2(k)}}\mathbb{E}[Q^d \circ Q^e(\mathbf{z}^{2^j})|\mathbf{z}]^{b_j(i)}\nonumber \\
	&\stackrel{[iii]}=\frac{1}{\sqrt{h}}f'(\mathbf{w}^T\textbf{v}(\mathbf{z})) \mathbf{z}^{\sum_{j=1}^{\ceil{\log_2(k)}} b_j(i)2^j}\nonumber \\
	&=\frac{1}{\sqrt{h}}f'(\mathbf{w}^T\textbf{v}(\mathbf{z}))\mathbf{z}^i={g_\mathbf{z}(\mathbf{w})}_i,
\end{align}
where: $[i]$ follows from the fact that conditioned on $\mathbf{z}$ each of the quantized values are independent; $[ii]$ follows from the fact  that $b_j(i) \in \{0,1\}$ and therefore, we either we take the expectation over the variable or have the expectation over 1; $[iii]$ follows from the fact that the quantized values using $Q^d\circ Q^e$ are unbiased estimators of their unquantized values.

The fact that $\|\widehat g_\mathbf{z}(\mathbf{w})\|_2\leq 1$ is obvious since the range of $Q^doQ^e$ is in $\{-1,1\}^k$.


\textbf{Lower bound:} For simplicity we only consider symmetric schemes. We will show that any symmetric quantization scheme that satisfies $\mathbb{E}[Q(\textbf{v}(\mathbf{z}))|\mathbf{z}]=\textbf{v}(\mathbf{z})$, $\forall \mathbf{z}\in \{\mathbf{z}\in \mathbb{R}| |\mathbf{z}|\leq 1\}$ uses at least $\log_2(d)$ bits. Consider a scheme that uses $k$ bits, hence, for any $\mathbf{z}$ with $|\mathbf{z}|\leq 1$, $Q(\textbf{v}(\mathbf{z}))$ takes one of $2^k$ values denoted as $e_1,...,e_{2^k}$. Let $d_i=\mathbb{E}[Q^d(e_i)]$, where $Q^d(e_i)$ is the decoded value of $e_i$. Then, $\mathbb{E}[Q(\textbf{v}(\mathbf{z}))|\mathbf{z}]$ is some convex combination of $d_1,...,d_{2^k}$. As a result, $\mathbb{E}[Q(\textbf{v}(\mathbf{z}))|\mathbf{z}]=\textbf{v}(\mathbf{z})$, $\forall \mathbf{z}\in \{\mathbf{z}\in \mathbb{R}| |\mathbf{z}|\leq 1\}$ implies that the set $\{\mathbf{z}\in \mathbb{R}| |\mathbf{z}|\leq 1\}$ is contained in the convex hull of the points $d_1,...,d_{2^k}$. 

Now consider the matrix with columns $[\mathbf{v}(\frac{1}{n}),\mathbf{v}(\frac{2}{n}),...,\mathbf{v}(1)]$. This matrix has full rank, and clearly $2^k$ is lower bounded by the rank of this matrix. Hence, $k\geq \log_2(d)$.

\section{Proof of Theorem \ref{th1}} \label{app:th1}
First, we reiterate Theorem~\ref{th1} from the paper for readability.

\begin{theorem*}
	The proposed $\mathtt{DataQ}$ algorithm satisfies the following statements: (1) For the integer vectors $(a,b)$ uniquely defining $\mathbf{z}_Q$, we have that $ (a,b) \in \mcal{S}$ with probability one; (2) $\mathtt{DataQ}$ uses at most $2\log_2(m)+\min\{2d\log_2(e\frac{2d+m^2}{2d}),m^2\log_2(e\frac{2d+m^2}{m^2})\}$ bits per sample for communication; (3) For the generated $\mathbf{z}_Q$, we have that $\|\mathbf{z}-{\mathbf{z}_Q}\|_\infty\leq \frac{B}{m-1}$ and $\|\mathbf{z}_Q\|_2\leq (1+\frac{\sqrt{d}}{m-1})B$.
\end{theorem*}
\begin{enumerate}
	\item \underline{\bf Proof that \textbf{$(a,b)\in \mcal{S}$:}}\\
	Recall that for datapoint $\mathbf{z}$, the definition of $(a,b)$ is given by 
	\begin{subequations}\label{eq:a_b_definition2}
		\begin{align}
	    \forall & j \in \{1,2,,\cdots,d\} :\nonumber \\
	& a_j(\mathbf{z})  \!=\!\!\argmax_{i \in \{0,1,\cdots,m-1\}}\!\!\!\bigg\{\!q_i\left|q_i = \frac{iB}{m-1}{\leq} |\mathbf{z}^+_j|\right\},\\
			&b_j(\mathbf{z})  \!=\!\!\argmax_{i \in \{0,1,\cdots,m-1\}}\!\!\!\bigg\{\!q_i\left|q_i = \frac{iB}{m-1}{\leq} |\mathbf{z}^-_j|\right\}.
		\end{align}
	\end{subequations}
	
	Thus, given the upper bound $B$ on $\|\mathbf{z}\|_2$, we have
	\begin{align}
		B^2&\geq \|\mathbf{z}\|_2^2=\|\mathbf{z}^+\|_2^2+\|\mathbf{z}^-\|_2^2 \nonumber \\
	&\stackrel{\eqref{eq:a_b_definition2}}\geq \sum_{j=1}^d\frac{B^2(a_j^2+b_j^2)}{(m-1)^2} \stackrel{(i)}\geq \sum_{j=1}^d\frac{B^2(a_j+b_j)}{(m-1)^2},
	\end{align}
	where $(i)$ follows since $a_j, b_j \in \mbb{N}$. From the above inequality, we now have that $$\sum_{j=1}^d (a_j+b_j)\leq (m-1)^2$$,
	which implies that $(a,b)\in \mcal{S} = \Big\{(\textbf{v}^{(1)},\textbf{v}^{(2)}) \in \mbb{N}^d \times \mbb{N}^d\ \Big|\ \|\textbf{v}^{(1)}\|_1 + \|\textbf{v}^{(2)}\|_1 \leq (m-1)^2 \Big\}$.
	
	\item \underline{\textbf{Proof of the upper bound on number of bits:}}\\
	It suffices to show that $\log_2(|\mcal{S}|)\leq 2\log_2(m)+\min\{d\log_2(e\frac{2d+m^2}{2d}),m^2\log_2(e\frac{2d+m^2}{m^2})\}$. 
	Note that $|\mcal{S}|$ can be written as 
	\begin{align}
	|\mcal{S}| = \sum_{q=0}^{(m-1)^2} &\left|\left\{ (\textbf{v}^{(1)},\textbf{v}^{(2)}) \in \mbb{N}^d \times \mbb{N}^d\right. \right.  \nonumber \\
	&\quad \left. \left. \Big|\ \|\textbf{v}^{(1)}\|_1 + \|\textbf{v}^{(2)}\|_1 \leq q\right\}\right|.  
	\end{align}
	For a given integer $q$, the number of positive integral solutions to the equation $\sum_{i=1}^d (\textbf{a}_i+\textbf{b}_i){=}q$ is a classical counting problem and its solution is given in closed form\footnote{The closed form relies on a standard approach in combinatorics called the ``stars and bars method''.} as $\binom{2d+q-1}{q}$~\cite{feller1957introduction}.
	Thus, we can write $|\mcal{S}|$ as
	\begin{align}
		&|\mathcal{S}| = \sum_{q=0}^{(m-1)^2}\binom{2d+q-1}{q} \leq \sum_{q=0}^{(m-1)^2}\binom{2d+q}{2d}\nonumber \\
	&\ \leq \sum_{q=0}^{(m-1)^2}\binom{2d+m^2}{2d}\nonumber \\
		&\ \stackrel{(i)}\leq  m^2 \min\left\{\binom{2d+m^2}{2d}, \binom{2d+m^2}{m^2}\right\}\nonumber\\ 
		&\ \stackrel{(ii)}\leq\!\! m^2\!\min\left\{\!\!\!\left(e\frac{2d+m^2}{2d}\right)^{2d}\!\!,\left(e\frac{2d+m^2}{m^2}\right)^{m^2}\!\right\},
	\end{align}
	where: $(i)$ is due to the symmetry of the binomial coefficient. Now by taking the logarithm of both sides, we get the intended upper bound for $\log_2(|\mcal{S}|)$; $(ii)$ uses the upper bound on the binomial coefficient based on Sterling's bounding of the factorial.
	
	\item \underline{\bf Proof that $|\textbf{z}_j-{\textbf{z}_Q}_j|\leq \frac{B}{m-1}$, \textbf{$\|\textbf{z}\|_2\leq (1+\frac{\sqrt{d}}{m-1})B$:}}\\
	$\|\textbf{z}_j-{\textbf{z}_Q}_j\|_\infty \leq \frac{B}{m-1}$ is directly due to the quantization scheme in \texttt{DataQ} since the distance between any two quantization levels is given by $\frac{B}{m-1}$.
	
	Now to prove the bound on $\|\textbf{z}\|_2$, note that
	\begin{align}
		&\|\textbf{z}_Q\|_2\stackrel{(i)}\leq \|\textbf{z}\|_2+\|\textbf{z}-\textbf{z}_Q\|_2\nonumber \\
		&\ \leq B+\sqrt{\|\textbf{z}-\textbf{z}_Q\|_2^2}= B+\sqrt{\sum_{j=1}^d |{\textbf{z}}_j-{\textbf{z}_Q}_j|^2}\nonumber \\
		&\ \leq B+\sqrt{\sum_{j=1}^d \frac{B^2}{(m-1)^2}}=\left(1+\frac{\sqrt{d}}{m-1}\right)B,
	\end{align}
	where $(i)$ follows from the triangle inequality.
\end{enumerate}

\section{Proof of Corollary~\ref{cor:convergence_no_gradcorr}}\label{app:proof_corollary_DataQ}
Here, we prove the the convergence guarantee when only $\mathtt{DataQ}$ is applied, which istated in Corollary~\ref{cor:convergence_no_gradcorr}.
The proof is similar to the SGD proof in~\cite{agarwal2018cpsgd,ghadimi2013stochastic}.

Throughout the proof, we use the following notaiton: (1) The maximum difference in risk $L(\textbf{w})$ is denoted as $D_L = \argmax_{\textbf{w}^{(1)}, \textbf{w}^{(2)} \in \mathcal{W}} | L(\textbf{w}^{(1)}) - L(\textbf{w}^{(2)})|$; (2) the upper bound on the gradient $\|\nabla L(\textbf{w})\|_2 \leq \widetilde{D},\ \forall \textbf{w} \in \mathcal{W}$.
We will assume that the learning rate used satisfies that $\eta \leq 1/C_w$. 

Additionally, recall that $\ell(\textbf{w},\textbf{z})$ is $C_w$-smooth in $\textbf{w}$ and $C_z$-smooth in $\textbf{z}$. 
Additionally, recall that at iteration $j$, we update $\textbf{w}^{(j)}$ using $g_{\textbf{z}_Q^{(j)}}(\textbf{w})$ which is the gradient computed from applying the model with parameters $\textbf{w}^{(j)}$ on the quantized version of $\textbf{z}^{(j)}$. 
Define, $\delta_j = g_{\textbf{z}_Q^{(j)}}(\textbf{w}^{(j)}) - \nabla L(\textbf{w}^{(j)})$.
Now, from the smoothness of $\ell(\textbf{w},\textbf{z})$, we have the following 
\begin{align}\label{eq:cor_eq1}
    &\mathbb{E}_{\textbf{z}^{(j)}}[L(\textbf{w}^{(j+1)}) {-} L(\textbf{w}^{(j)})] 
    \nonumber \\
	&{\leq} \nabla\! L(\textbf{w}^{(j)})^T\!\mathbb{E}_{\textbf{z}^{(j)}}\![\textbf{w}^{(j+1)}{-}\textbf{w}^{(j)}] {+} \frac{C_w}{2}\mathbb{E}_{\textbf{z}^{(j)}}\![ \|\textbf{w}^{(j+1)}{-}\textbf{w}^{(j)}\|_2^2] \nonumber \\
    &= {-} \eta\!\nabla\! L(\!\textbf{w}^{(j)}\!)^T\!\mathbb{E}_{\textbf{z}^{(j)}}\!\!\left[\!g_{\textbf{z}_Q^{(j)}}(\!\textbf{w}^{(j)}\!)\!\right] \!{+} \frac{C_w \eta^2}{2} \mathbb{E}_{\textbf{z}^{(j)}}\!\!\left[\! \left\|g_{\textbf{z}_Q^{(j)}}(\!\textbf{w}^{(j)}\!)\right\|_2^2\!\right] \nonumber \\
    &= - \eta \left\|\nabla L(\textbf{w}^{(j)})\right\|_2^2 - \eta\nabla L(\textbf{w}^{(j)})^T\mathbb{E}_{\textbf{z}^{(j)}}\left[\delta_j\right] \nonumber \\
	&\hspace{0.2in}+ \frac{C_w \eta^2}{2} \mathbb{E}_{\textbf{z}^{(j)}}\left[ \left\|\nabla L(\textbf{w}^{(j)}) + \delta_j\right\|_2^2\right] \nonumber \\
    &= - \eta \left\|\nabla L(\textbf{w}^{(j)})\right\|_2^2 - \eta\nabla L(\textbf{w}^{(j)})^T\mathbb{E}_{\textbf{z}^{(j)}}\left[\delta_j\right] \nonumber \\
    &\hspace{0.2in}+ \frac{C_w \eta^2}{2} \mathbb{E}_{\textbf{z}^{(j)}}\!\!\left[\left\|\nabla L(\textbf{w}^{(j)})\right\|_2^2 {+} 2\nabla L(\textbf{w}^{(j)})^T \delta_j {+} \left\|\delta_j\right\|_2^2 \right] \nonumber \\    
    &= - \eta \left(1-\frac{\eta C_w}{2}\right)\left\|\nabla L(\textbf{w}^{(j)})\right\|_2^2 \nonumber \\
	&\hspace{0.1in}+\eta\!\left(1{-}\eta C_w\right)\!\nabla L(\textbf{w}^{(j)})^T\!\mathbb{E}_{\textbf{z}^{(j)}}\!\!\left[\delta_j\right] {+} \frac{C_w \eta^2}{2}\! \mathbb{E}_{\textbf{z}^{(j)}}\!\!\left[\!\left\|\delta_j\!\right\|_2^2\!\right].
\end{align}
Reorganizing~\eqref{eq:cor_eq1}, we get that
\begin{align}
    &\|\nabla L(\textbf{w}^{(j)})\|_2^2 \nonumber \\
    &\leq \frac{2}{\eta(2\!-\!\eta C_w)}\left(L(\!\textbf{w}^{(j)}\!) {-} L(\!\textbf{w}^{(j+1)}) {+} \frac{\eta^2C_w}{2}\mathbb{E}_{\textbf{z}^{(j)}}\!\!\left[\|\delta_j\|_2^2\right]\right) \nonumber \\
    &\qquad + \frac{2\left(1-C_w \eta\right)}{\left(2-C_w \eta\right)}\nabla L(\textbf{w}^{(j)})^T\mathbb{E}_{\textbf{z}^{(j)}}\left[\delta_j\right] \nonumber \\
    &\leq \frac{2}{\eta(2{-}\eta C_w)}\left(L(\!\textbf{w}^{(j)}\!) - L(\!\textbf{w}^{(\star)}\!)
    + \frac{\eta^2C_w}{2}\mathbb{E}_{\textbf{z}^{(j)}}\left[\|\delta_j\|_2^2\right]\right)\nonumber \\ 
    &\qquad + \frac{2\left(1-C_w \eta\right)}{\left(2-C_w \eta\right)}\nabla L(\textbf{w}^{(j)})^T\mathbb{E}_{\textbf{z}^{(j)}}\left[\delta_j\right].
	\end{align}
Hence,
\begin{align}\label{eq:cor_eq2}
    &\|\nabla L(\textbf{w}^{(j)})\|_2^2 \nonumber \\
    &\stackrel{(a)}\leq \frac{2}{\eta(2-\eta C_w)}\left(L(\textbf{w}^{(j)}) - L(\textbf{w}^{(\star)})
	+ \frac{\eta^2C_w}{2}\mathbb{E}_{\textbf{z}^{(j)}}\left[\|\delta_j\|_2^2\right]\right) \nonumber \\
    &\qquad + \frac{2\left(1-C_w \eta\right)}{\left(2-C_w \eta\right)}\left\|\nabla L(\textbf{w}^{(j)})\right\|_2\left\|\mathbb{E}_{\textbf{z}^{(j)}}\left[\delta_j\right]\right\|_2,
	\end{align}
where, $(a)$ follows from the Cauchy-Schwartz inequality and the fact that $\eta \leq 1/C_w$.
Now, we note that
\begin{align}\label{eq:cor_eq_delta}
&\left\|\mathbb{E}_{\textbf{z}^{(j)}}\!\!\left[\delta_j\right]\right\|_2 = \left\|\mathbb{E}_{\textbf{z}^{(j)}}\!\!\left[g_{\textbf{z}_Q^{(j)}}(\textbf{w}^{(j)})\right] - \nabla L(\textbf{w}^{(j)}) \right\|_2 \nonumber\\ 
&= \left\|\mathbb{E}_{\textbf{z}^{(j)}}\!\!\left[g_{\textbf{z}_Q^{(j)}}(\textbf{w}^{(j)}) - g_{\textbf{z}^{(j)}}(\textbf{w}^{(j)})\right] \right\|_2 \nonumber\\ 
&\leq \mathbb{E}_{\textbf{z}^{(j)}}\!\!\left[\!\left\|g_{\textbf{z}_Q^{(j)}}(\!\textbf{w}^{(j)}\!) - g_{\textbf{z}^{(j)}}(\!\textbf{w}^{(j)}\!) \right\|_2\right] \stackrel{(a)}\leq C_z B \frac{\sqrt{d}}{(m-1)}.
\end{align}
where $(a)$ follows from the bound in~\eqref{eq:bounded_bias}. Additionally, we also have that
\begin{align}\label{eq:cor_eq_delta_2}
&\mathbb{E}_{\textbf{z}^{(j)}}\!\!\left[\left\|\delta_j\right\|_2^2\right] \nonumber \\
&= \mathbb{E}_{\textbf{z}^{(j)}}\!\!\left[\!\left\|\!\left(\!g_{\textbf{z}_Q^{(j)}}(\!\textbf{w}^{(j)}\!) {-} g_{\textbf{z}^{(j)}}(\!\textbf{w}^{(j)}\!)\!\right) \!{-}\! \left(\!g_{\textbf{z}^{(j)}}(\!\textbf{w}^{(j)}\!) {-} \nabla L(\!\textbf{w}^{(j)}\!)\right)\! \right\|_2^2 \right] \nonumber \\
&\leq 3\mathbb{E}_{\textbf{z}^{(j)}}\!\!\!\left[\!\left\|g_{\textbf{z}_Q^{(j)}}(\!\textbf{w}^{(j)}\!) {-} g_{\textbf{z}^{(j)}}(\!\textbf{w}^{(j)}\!) \!\right\|_2^2\! + \!\left\|g_{\textbf{z}^{(j)}}(\!\textbf{w}^{(j)}\!) {-} \nabla L(\!\textbf{w}^{(j)}\!) \!\right\|_2^2\! \right] \nonumber\\ 
&\stackrel{(a)}\leq 3 \left(C_z B \frac{d}{(m-1)^2} + \sigma^2\right),
\end{align}
where $\sigma^2$ is the bounded variance of the unquantized stochastic gradient and $(a)$ follows from the bound in~\eqref{eq:bounded_bias}.

By setting $m = h \sqrt{d}$ in~\eqref{eq:cor_eq_delta} and~\eqref{eq:cor_eq_delta_2} and substituting in~\eqref{eq:cor_eq2} we have that 
\begin{align}
    &\|\nabla L(\textbf{w}^{(j)})\|_2^2 \nonumber \\
    &\leq \frac{2}{\eta(2{-}\eta C_w)}\!\Bigg(L(\!\textbf{w}^{(j)}\!) {-} L(\!\textbf{w}^{(\star)}\!) 
	+ \eta^2 C_w\overbrace{\frac{3}{2}\left(\frac{2 C_z B}{h^2}{+} \sigma^2\right)}^{\widehat{\sigma}^2}\Bigg) \nonumber \\
    &\qquad  + \frac{4\left(1-C_w \eta\right)}{\left(2-C_w \eta\right)}\left\|\nabla L(\textbf{w}^{(j)})\right\|_2\frac{C_z B}{h}.
\end{align}
By averaging across $j \in [n]$, we get
\begin{align}
    &\frac{1}{n}\sum_{j=1}^n\left\|\nabla L(\textbf{w}^{(j)})\right\|_2^2  \nonumber \\
    &\leq \frac{2}{n\eta(2-\eta C_w)} \left(L(\textbf{w}^{(0)})- L(\textbf{w}^\star) + n\ \eta^2 C_w \widehat{\sigma}^2 \right) \nonumber \\
    &+ \frac{4\left(1-C_w \eta\right)}{\left(2-C_w \eta\right)}\frac{C_z B}{h} \frac{1}{n}\sum_{j=1}^n\left\|\nabla L(\textbf{w}^{(j)})\right\|_2.
\end{align}
Hence,
\begin{align}
    &\frac{1}{n}\sum_{j=1}^n\left\|\nabla L(\textbf{w}^{(j)})\right\|_2^2  \nonumber \\
&\leq \frac{2}{n\eta(2{-}\eta C_w)} \left(\! \left(\!L(\!\textbf{w}^{(0)}\!){-} L(\!\textbf{w}^\star\!)\!\right)  {+} n\ \eta^2 C_w \widehat{\sigma}^2 \!\right) {+} \frac{2C_z B}{h} \widetilde{D} \nonumber \\
    &\leq \frac{2}{n\eta(2-\eta C_w)} \left( D_L  + n\ \eta^2 C_w \widehat{\sigma}^2 \right) + \frac{2C_z B}{h} \widetilde{D} \nonumber \\
    &\stackrel{(a)}\leq \frac{2}{n\eta} \left( D_L  + n\ \eta^2 C_w \widehat{\sigma}^2 \right) + \frac{2C_z B}{h} \widetilde{D} \nonumber \\
    &= \frac{2}{n\eta}  D_L + 2\ \eta C_w \widehat{\sigma}^2 + \frac{2C_z B}{h} \widetilde{D},
\end{align}
where $(a)$ follows from assuming $\eta \leq 1/C_w$.
By setting $\eta=\min\{\frac{1}{C_w}, \frac{\sqrt{2D_L}}{\widehat{\sigma}\sqrt{C_w n}}\}$, we get that
\begin{align}
    &\frac{1}{n}\sum_{j=1}^n\left\|\nabla L(\textbf{w}^{(j)})\right\|_2^2  \nonumber \\
    &\leq \frac{2D_L}{n} \max\left\{C_w,\frac{\widehat{\sigma}\sqrt{C_w n}}{\sqrt{2D_L}}\right\} \nonumber \\
	&\quad + 2\ C_w \widehat{\sigma}^2 \min\left\{\frac{1}{C_w}, \frac{\sqrt{2D_L}}{\widehat{\sigma}\sqrt{C_w n}}\right\} + \frac{2C_z B}{h} \widetilde{D} \nonumber \\
    &\leq \frac{2D_LC_w}{n} + \frac{\widehat{\sigma}\sqrt{2D_L C_w}}{\sqrt{n}} +  2\frac{\widehat{\sigma}\sqrt{2D_LC_w}}{\sqrt{n}} + \frac{2C_z B}{h} \widetilde{D} \nonumber \\
    &\leq \frac{2D_LC_w}{n} {+}  3\frac{\widehat{\sigma}\sqrt{2D_LC_w}}{\sqrt{n}} {+} \frac{2C_z B}{h} \widetilde{D} = O(\frac{1}{\sqrt{n}} {+} \frac{1}{h}).
\end{align}
This concludes the proof of Corollary~\ref{cor:convergence_no_gradcorr}.

\section{Proof of Lemma \ref{lem:grad_unbiased_estimate}}\label{app:lem2}
\begin{lemma*}
	The quantized stochastic gradient $\widehat{g}_\textbf{z}(\textbf{w})$ is an unbiased estimate of  $ \nabla L(\textbf{w})$ and satisfies $\|\widehat{g}_\textbf{z}(\textbf{w})-\nabla L(\textbf{w})\|_2\leq C_zB\left(2 + (h+1)\frac{\sqrt{d}}{m-1}\right)$.
\end{lemma*}
\begin{proof}
	Recall that $\Delta = g_\textbf{z}(\textbf{w}) - g_{\textbf{z}_Q}(\textbf{w})$, and that, we have
	\begin{equation}\label{eq:Delta_hat_defin}
		\widehat \Delta_i=\begin{cases}
			{0}   &   {\text{with probability } 1-\frac{1}{h}} \\
			{C_zBh\frac{\sqrt{d}}{m-1}}   &  {\text{with probability } \frac{\Delta_i+C_z B\frac{\sqrt{d}}{m-1}}{2C_z B\frac{\sqrt{d}}{m-1}}\frac{1}{h}}\\
			{-C_zBh\frac{\sqrt{d}}{m-1}} & {\text{with probability } \frac{C_z B\frac{\sqrt{d}}{m-1}-\Delta_i}{2C_z C\frac{\sqrt{d}}{m-1}}\frac{1}{h}.}
		\end{cases}
	\end{equation}
	Hence, we can prove the unbiasedness property of $\widehat{\Delta}$ by direct computation as follows
	\begin{align}\label{eq:Delta_unbias}
		\mathbb{E}[{\widehat \Delta}_i|\Delta] &= \mathbb{E}[{\widehat \Delta}_i|\Delta_i]\nonumber \\
		&= 0 + \left({C_zBh\frac{\sqrt{d}}{m-1}}\right)\frac{\Delta_i+C_z B\frac{\sqrt{d}}{m-1}}{2C_z B\frac{\sqrt{d}}{m-1}}\frac{1}{h} \nonumber \\
	&\quad + \left(-{C_zBh\frac{\sqrt{d}}{m-1}}\right)\frac{C_z B\frac{\sqrt{d}}{m-1}-\Delta_i}{2C_z C\frac{\sqrt{d}}{m-1}}\frac{1}{h}\nonumber \\
		&= \left(2{C_zBh\frac{\sqrt{d}}{m-1}}\right)\frac{2\Delta_i}{2C_z B\frac{\sqrt{d}}{m-1}h} = \Delta_i.
	\end{align}
	
	Next, we can prove the unbiasedness of $\widehat g_\textbf{z}(\textbf{w})$ as follows
	\begin{align}
		\mathbb{E}[\widehat g_\textbf{z}(\textbf{w})|\textbf{z}]&=\mathbb{E}_{\textbf{z}_Q}\left[\mathbb{E}[\widehat g_\textbf{z}(\textbf{w})|\textbf{z},\textbf{z}_Q]\right]
		\nonumber \\
	&\stackrel{(i)}=\mathbb{E}_{\textbf{z}_Q}\left[\mathbb{E}[g_{\textbf{z}_Q}(\textbf{w})+\widehat \Delta|\textbf{z},\textbf{z}_Q]\right]\nonumber\\
		&\stackrel{(ii)}=\mathbb{E}_{\textbf{z}_Q}\left[\mathbb{E}[g_{\textbf{z}_Q}(\textbf{w})+\mbb{E}[\widehat \Delta|\Delta,\textbf{z},\textbf{z}_Q]|\textbf{z},\textbf{z}_Q]\right]\nonumber\\
		&\stackrel{(iii)}=\mathbb{E}_{\textbf{z}_Q}\left[\mathbb{E}[g_{\textbf{z}_Q}(\textbf{w})+\Delta|\textbf{z},\textbf{z}_Q]\right]\nonumber \\
		&\stackrel{(iv)}=\mathbb{E}_{\textbf{z}_Q}\left[\mathbb{E}[g_\textbf{z}(\textbf{w})|\textbf{z},\textbf{z}_Q]\right]\nonumber\\
		&=g_\textbf{z}(\textbf{w}),
	\end{align}
	where: $(i)$ follows from the definition of $\widehat g_\textbf{z}(\textbf{w})$; $(ii)$ follows by the tower property of expectation and the fact that $\widehat \Delta$ is constructed by stochastic quantization of $\Delta$;
	$(iii)$ follows from the unbiasedness of $\widehat{\Delta}$ in~\eqref{eq:Delta_unbias};
	$(iv)$  is due to the fact that $\Delta = g_\textbf{z}(\textbf{w}) - g_{\textbf{z}_Q}(\textbf{w})$.
	
	To prove the bound on the variance, we note that
	\begin{align}
		\|&\widehat g_\textbf{z}(\textbf{w})-\nabla L(\textbf{w})\|_2\stackrel{(i)}\leq \|g_{\textbf{z}_Q}(\textbf{w})-\nabla L(\textbf{w})\|_2+\|\widehat \Delta\|_2\nonumber \\
		&\stackrel{(i)}\leq \|g_\textbf{z}(\textbf{w})-\nabla L(\textbf{w})\|_2+\|g_\textbf{z}(\textbf{w})-g_{\textbf{z}_Q}(\textbf{w})\|_2+\|\widehat{\Delta}\|_2\nonumber\\
		&\stackrel{(ii)}\leq \|g_\textbf{z}(\textbf{w})-\nabla L(\textbf{w})\|_2+C_z B\frac{\sqrt{d}}{m-1}+\|\widehat{\Delta}\|_2 \nonumber\\
		&\stackrel{(iii)}\leq \|g_\textbf{z}(\textbf{w})-\nabla L(\textbf{w})\|_2+C_z B\frac{\sqrt{d}}{m-1}+C_zBh\frac{\sqrt{d}}{m-1} \nonumber\\
		&= \|\frac{1}{N} \sum_{i=1}^N (g_{\textbf{z}^{(i)}}(\textbf{w})-g_\textbf{z}(\textbf{w}))\|_2\nonumber \\
	&\quad +C_z B\frac{\sqrt{d}}{m-1}+C_zBh\frac{\sqrt{d}}{m-1} \nonumber \\ 
		&\stackrel{(iv)}\leq \frac{1}{N}\sum_{i=1}^N\|g_{\textbf{z}^{(i)}}(\textbf{w})-g_\textbf{z}(\textbf{w})\|_2\nonumber \\
	&\quad +C_z B\frac{\sqrt{d}}{m-1}+C_zBh\frac{\sqrt{d}}{m-1} \nonumber \\
		&\stackrel{(v)}\leq \frac{1}{N}\sum_{i=1}^N C_z\|\textbf{z}^{(i)} - \textbf{z}\|_2 +C_z B\frac{\sqrt{d}}{m-1}+C_zBh\frac{\sqrt{d}}{m-1} \nonumber\\
		&\leq 2 C_z B +C_z B\frac{\sqrt{d}}{m-1}+C_zBh\frac{\sqrt{d}}{m-1},
	\end{align}
	where: $(i)$ is due to the triangle inequality;
	$(ii)$ follows from the $C_z$-Lipschitz continuity of $g_\textbf{z}$ and the definition of \texttt{DataQ};
	$(iii)$ follows from~\eqref{eq:Delta_hat_defin};
	$(iv)$ follows due to the convexity of the norm; $(v)$ follows from the $C_z$-Lipschitz continuity of $g_\textbf{z}(\textbf{w})$ in $\textbf{z}$.
\end{proof}
\section{Proof of Theorem \ref{thm:thresholding}}\label{app:sampsel}
In this part we prove the convergence of SGD when our sample selection scheme through thresholding is applied. We denote the stochastic gradient with loss thresholding at the $j$-th iteration to be $\widetilde{g}_{\textbf{z}^{(j)}}(\textbf{w}^{(j)})$ which is given by
\begin{align*}
\widetilde{g}_{\textbf{z}^{(j)}}(\textbf{w}^{(j)}) = \widehat{g}_{\textbf{z}^{(j)}}(\textbf{w}^{(j)}) \mbb{I}\left(\ell(\textbf{w}^{(j)},\textbf{z}^{(i)})\leq \ell_{th}^{(j)}\right),
\end{align*}
where $\mbb{I}(.)$ is the indicator function.
In the following we present a simple proof of our result by adapting a standard proof of SGD convergence. From convexity of $L(\textbf{w})$, we have that
\begin{equation}\label{conv}
	L(\textbf{w}^{(j)})\leq L(\textbf{w}^*)+\nabla L(\textbf{w}^{(j)})^T(\textbf{w}^{(j)}-\textbf{w}^*).
\end{equation} 
It is well known that for convex function, smoothness implies a quadratic upper bound \cite{bubeck2014convex}, i.e., we have that 
\begin{align}
	L(\textbf{w}^{(j+1)})&\leq L(\textbf{w}^{(j)})+\nabla L(\textbf{w}^{(j)})^T(\textbf{w}^{(j+1)}-\textbf{w}^{(j)})\nonumber \\
	&\quad +\frac{C_w}{2}\|\textbf{w}^{(j+1)}-\textbf{w}^{(j)}\|_2^2\nonumber \\
	&\stackrel{(a)}= L(\textbf{w}^{(j)})-\eta \nabla L(\textbf{w}^{(j)})^T(\widetilde g_{\textbf{z}^{(j)}}(\textbf{w}^{(j)}))\nonumber \\
	&\quad +\frac{\eta^2 C_w}{2}\|\widetilde g_{\textbf{z}^{(j)}}(\textbf{w}^{(j)})\|_2^2,
\end{align}
where $(a)$ follows from the definition of the model update rule using SGD in~\eqref{eq:update_rule_SGD} with gradient $\widetilde g_{\textbf{z}^{(j)}}(\textbf{w}^{(j)})$.

Taking the expectation of both sides over the randomness in $\widetilde g_{\textbf{z}^{(j)}}(\textbf{w}^{(j)})$, we get
\begin{align}\label{eq:thm3_expectation_taken}
	\mbb{E}\left[L(\textbf{w}^{(j+1)})\right]
	&\leq L(\textbf{w}^{(j)})-\eta \nabla L(\textbf{w}^{(j)})^T(\mbb{E}\left[\widetilde g_{\textbf{z}^{(j)}}(\textbf{w}^{(j)})\right])\nonumber \\
	&\quad +\frac{\eta^2 C_w}{2}\mbb{E}\left[\|\widetilde g_{\textbf{z}^{(j)}}(\textbf{w}^{(j)})\|_2^2\right],
\end{align}
Note that by definition of $\widetilde g_{\textbf{z}^{(j)}}(\textbf{w}^{(j)})$, we have that the expectation of the computed gradient is $\mathbb{E}\left[\widetilde g_{\textbf{z}^{(j)}}(\textbf{w}^{(j)})\right] = \nabla L(\textbf{w}^{(j)})- \Delta_j$, where $\Delta_j =\frac{1}{N} \sum_{i=1}^N\nabla \ell(\textbf{w},\textbf{z}^{(i)})\mbb{I}\left(\ell(\textbf{w}^{(j)},\textbf{z}^{(i)})\leq \ell_{th}^{(j)}\right)$. Additionally, we have that
\begin{align}\label{eq:second_moment_thresh}
\mbb{E}\left[\left(\widetilde{g}_{\textbf{z}^{(j)}}(\textbf{w}^{(j)}\right)^2\right] 
\leq 
\mbb{E}\left[\left(\widehat{g}_{\textbf{z}^{(j)}}(\textbf{w}^{(j)}\right)^2\right] = \|\nabla L(\textbf{w})\|_2^ 2 + \widetilde{B}^2,
\end{align}
Thus, by substituting this bounds on first and second moments of $\widetilde g_{\textbf{z}^{(j)}}$ in~\eqref{eq:thm3_expectation_taken}, and using the fact that $\eta \leq \frac{1}{C_w}$, we have that
\begin{align}\label{quad}
	\mathbb{E}&\left[L(\textbf{w}^{(j+1)})\right]\leq L(\textbf{w}^{(j)})-\eta \nabla L(\textbf{w}^{(j)})^T(\nabla L(\textbf{w}^{(j)})- \Delta_j)\nonumber \\
	&\quad +\frac{\eta^2 C_w}{2}\| \nabla L(\textbf{w}^{(j)})- \Delta_j\|_2^2+ \frac{\eta^2 C_w}{2}\widetilde{B}^2\nonumber \\
	&\quad -\frac{\eta^2 C_w}{2}\|\Delta_j\|_2^2+\eta^2 C_w \Delta_j^T\nabla L(\textbf{w}^{(j)}) \nonumber \\
	=& L(\textbf{w}^{(j)})-\eta(1-\frac{\eta C_w}{2})\| \nabla L(\textbf{w}^{(j)})- \Delta_j\|_2^2+ \frac{\eta^2 C_w}{2}\widetilde{B}^2\nonumber \\
	&\quad -\eta(1-\eta C_w) \Delta_j^T\nabla L(\textbf{w}^{(j)})+\eta (1-\frac{\eta C_w}{2}) \|\Delta_j\|_2^2\nonumber \\
	\leq& L(\textbf{w}^{(j)})-\frac{\eta}{2}\| \nabla L(\textbf{w}^{(j)})- \Delta_j\|_2^2\nonumber \\
	& -\eta(1-\eta C_w) \Delta_j^T\nabla L(\textbf{w}^{(j)}) + \frac{\eta}{2}(\widetilde{B}^2+2\|\Delta_j\|_2^2).
\end{align}
Combining \eqref{conv}, \eqref{quad}, we get
\begin{align}
	\mathbb{E}&\left[L(\textbf{w}^{(j+1)})\right]\leq L(\textbf{w}^*)+\nabla L(\textbf{w}^{(j)})^T(\textbf{w}^{(j)}-\textbf{w}^*)\nonumber \\
	&-\frac{\eta}{2}\| \nabla L(\textbf{w}^{(j)})- \Delta_j\|_2^2+ \frac{\eta}{2}(\widetilde{B}^2+2\|\Delta_j\|_2^2)\nonumber \\
	&-\eta(1-\eta C_w) \Delta_j^T\nabla L(\textbf{w}^{(j)}).
\end{align}
By adding and subtracting $\frac{1}{2\eta}\|\textbf{w}^{(j)}-\textbf{w}^*\|_2^2$ to complete square we get that
\begin{align}\label{eq:Thm3_SGD}
	\mathbb{E}&\left[L(\textbf{w}^{(j+1)})\right]\leq L(\textbf{w}^*)+\frac{1}{2\eta}\|\textbf{w}^{(j)}-\textbf{w}^*\|_2^2\nonumber \\
	&-\frac{1}{2\eta}\|\textbf{w}^{(j)}-\textbf{w}^*-\eta(\nabla L(\textbf{w}^{(j)})- \Delta_j)\|_2^2\nonumber \\
	&+\Delta_j^T(\textbf{w}^{(j)}-\textbf{w}^*)-\eta(1-\eta C_w) \Delta_j^T\nabla L(\textbf{w}^{(j)})\nonumber \\
	&+ \frac{\eta}{2}(\widetilde{B}^2+2\|\Delta_j\|_2^2)\nonumber \\
	\leq &L(\textbf{w}^*)-\frac{1}{2\eta}\mathbb{E}\left[\|\textbf{w}^{(j)}-\textbf{w}^*-\eta \widetilde g_{\textbf{z}^{(j)}}(\textbf{w}^{(j)})\|_2^2\right]\nonumber \\
	&+\frac{1}{2\eta}\|\textbf{w}^{(j)}-\textbf{w}^*\|_2^2+\Delta_j^T(\textbf{w}^{(j)}-\textbf{w}^*)\nonumber \\
	&+\eta^2 C_w \Delta_j^T\nabla L(\textbf{w}^{(j)})+ \eta (\widetilde{B}^2+\frac{1}{2}\|\Delta_j\|_2^2)\nonumber \\
	\stackrel{(a)}=&L(\textbf{w}^*)-\frac{1}{2\eta}\mathbb{E}\left[\|\textbf{w}^{(j+1)}-\textbf{w}^*\|_2^2\right]+\frac{1}{2\eta}\|\textbf{w}^{(j)}-\textbf{w}^*\|_2^2\nonumber \\
	&+\Delta_j^T(\textbf{w}^{(j)}-\textbf{w}^*)+\eta^2 C_w \Delta_j^T\nabla L(\textbf{w}^{(j)})\nonumber \\
	&+ \eta (\widetilde{B}^2+\frac{1}{2}\|\Delta_j\|_2^2)\nonumber \\
	\stackrel{(b)}\leq &L(\textbf{w}^*)-\frac{1}{2\eta}\mathbb{E}\left[\|\textbf{w}^{(j+1)}-\textbf{w}^*\|_2^2-\|\textbf{w}^{(j)}-\textbf{w}^*\|_2^2\right]\nonumber \\
	&+\|\Delta_j\|_2\|\textbf{w}^{(j)}-\textbf{w}^*\|_2+\eta \|\Delta_j\|_2\|\nabla L(\textbf{w}^{(j)})\|_2\nonumber \\
	&+ \eta (\widetilde{B}^2+\frac{1}{2}\|\Delta_j\|_2^2) \nonumber \\
    =&L(\textbf{w}^*)-\frac{1}{2\eta}\mathbb{E}\left[\|\textbf{w}^{(j+1)}-\textbf{w}^*\|_2^2-\|\textbf{w}^{(j)}-\textbf{w}^*\|_2^2\right]+\eta \widetilde{B}^2\nonumber \\	
	&+\|\Delta_j\|_2\left(\|\textbf{w}^{(j)}-\textbf{w}^*\|_2+\eta \|\nabla L(\textbf{w}^{(j)})\|_2 +\frac{\eta}{2}\|\Delta_j\|_2\right),	
\end{align}
where: $(a)$ followed from the update rule of $\textbf{w}^{(j+1)}$ from $\textbf{w}^{(j)}$ and $g_{\textbf{z}^{(j)}}(\textbf{w}^{(j)})$; $(b)$ from Cauchy-Schwartz inequality.
We now want to find bounds for the terms 
$\|\Delta_j\|_2$, $\|\textbf{w}^{(j)}-\textbf{w}^*\|_2$ and $\|\nabla L(\textbf{w}^{(j)})\|_2$.
Note that from our assumptions, we have that $\|\textbf{w}^{(j)}-\textbf{w}^*\|_2 \leq \widetilde{D}$. 

For a point $\textbf{z}$, let $\textbf{w}'_z$ be the minimizer of $\ell(.,\textbf{z})$, then we have that
\begin{align*}
\|\nabla \ell(\textbf{w},\textbf{z})&-\nabla \ell(\textbf{w}_z',\textbf{z})\|_2^2 \stackrel{(a)}\leq 2C_w\left(\nabla \ell(\textbf{w}_z',\textbf{z})^T(\textbf{w}_z'-\textbf{w})\right. \nonumber \\
&\qquad \qquad \qquad \qquad \left. +\ell(\textbf{w},\textbf{z})-\ell(\textbf{w}_z',\textbf{z})\right)\\ 
&\stackrel{(b)}\leq 2C_w\left(\nabla \ell(\textbf{w}_z',\textbf{z})^T(\textbf{w}_z'-\textbf{w})+\ell(\textbf{w},\textbf{z})\right), 
\end{align*}
where $(a)$ is a consequence of the smoothness of $\ell(\textbf{w},\textbf{z})$ and $(b)$ follows from our assumption that $\ell(\textbf{w},\textbf{z}) \geq 0$.
By making use of the above equation and the fact that $\nabla \ell(\textbf{w}_z',\textbf{z}) =0$ for a given $\textbf{z}$, we get that $\|\nabla \ell(\textbf{w},\textbf{z})\|_2^2\leq 2C_w L(\textbf{w},\textbf{z})$. 
Using this bound on $\|\nabla \ell(\textbf{w},\textbf{z})\|_2^2$, we have that
\begin{align}\label{eq:delta_bound}
    \|\Delta_j\|_2 &=  \left\|\frac{1}{N} \sum_{i=1}^N\nabla \ell(\textbf{w},\textbf{z}^{(i)})\mbb{I}\left(\ell(\textbf{w}^{(j)},\textbf{z}^{(i)})\leq \ell_{th}^{(j)}\right) \right\|_2 \nonumber \\
    &\leq  \frac{1}{N} \sum_{i=1}^N \left\|\nabla \ell(\textbf{w},\textbf{z}^{(i)})\right\|_2 \mbb{I}\left(\ell(\textbf{w}^{(j)},\textbf{z}^{(i)})\leq \ell_{th}^{(j)}\right)  \nonumber \\    
     &\leq  \frac{1}{N} \sum_{i=1}^N \left(\sqrt{2 C_w \ell(\textbf{w},\textbf{z}^{(i)})}\right)\mbb{I}\left(\ell(\textbf{w}^{(j)},\textbf{z}^{(i)})\leq \ell_{th}^{(j)}\right)  \nonumber \\
     &\leq \sqrt{2 C_w \ell^{(j)}_{th}}.
\end{align}

Finally, note that the smoothness property implies that $\|\nabla \ell(\textbf{w},\textbf{z})\|_2\leq 2C_w\widetilde{D}$, which using the same steps as in~\eqref{eq:delta_bound} implies that $\|\Delta_j\|_2 \leq 2 C_w \widetilde{D}$.
Additionally, it also implies that 
$\|\nabla L(\textbf{w})\|_2 \leq 2 C_w \widetilde{D}$.

By substituting these upper bounds in~\eqref{eq:Thm3_SGD}, we get that
\begin{align}
	\mathbb{E}&\left[L(\textbf{w}^{(j+1)})\right]\leq 
	L(\textbf{w}^*)-\frac{1}{2\eta}\mathbb{E}\left[\|\textbf{w}^{(j+1)}-\textbf{w}^*\|_2^2\right. \nonumber \\
     &\quad \left. -\|\textbf{w}^{(j)}-\textbf{w}^*\|_2^2\right]+\eta \widetilde{B}^2\nonumber \\
		&\quad +\|\Delta_j\|_2\left(\|\textbf{w}^{(j)}-\textbf{w}^*\|_2+\eta \|\nabla L(\textbf{w}^{(j)})\|_2 +\frac{\eta}{2}\|\Delta_j\|_2\right) \nonumber \\
    &=
	L(\textbf{w}^*)-\frac{1}{2\eta}\mathbb{E}\left[\|\textbf{w}^{(j+1)}-\textbf{w}^*\|_2^2-\|\textbf{w}^{(j)}-\textbf{w}^*\|_2^2\right]\nonumber \\
		&\quad +\left(\sqrt{2 C_w \ell^{(j)}_{th}}\right)\left(\widetilde{D}+2 \eta C_w \widetilde{D} +\eta C_w \widetilde{D}\right) +\eta \widetilde{B}^2 \nonumber \\	
	&= L(\textbf{w}^*)-\frac{1}{2\eta}\mathbb{E}\left[\|\textbf{w}^{(j+1)}-\textbf{w}^*\|_2^2-\|\textbf{w}^{(j)}-\textbf{w}^*\|_2^2\right]\nonumber \\
	&\quad +\widetilde{D}\sqrt{2C_w \ell_{th}^{(j)}}(1+3\eta C_w) +\eta \widetilde{B}^2.
\end{align}
And from convexity and $\sum_{j=1}^n\sqrt{\ell_{th}^{(j)}}\leq \sqrt{n}$, we have 
\begin{align}
	\mathbb{E}\left[\frac{1}{n}\sum_{j=1}^nL(\textbf{w}^{(j)})\right]\leq &L(\textbf{w}^*)+\frac{\|\textbf{w}^{(0)}-\textbf{w}^*\|_2^2}{2n\eta}\nonumber \\
     &+\frac{ \widetilde{D}\sqrt{2C_w}(1+3\eta C_w)}{\sqrt{n}} +\eta \widetilde{B}^2.
\end{align}
Substituting $\eta = \frac{1}{\widetilde{B}\sqrt{n}}$, we get the result.

\bibliographystyle{IEEEtran}
\bibliography{QuantNIPS}

\begin{IEEEbiography}[{\includegraphics[width=1in]{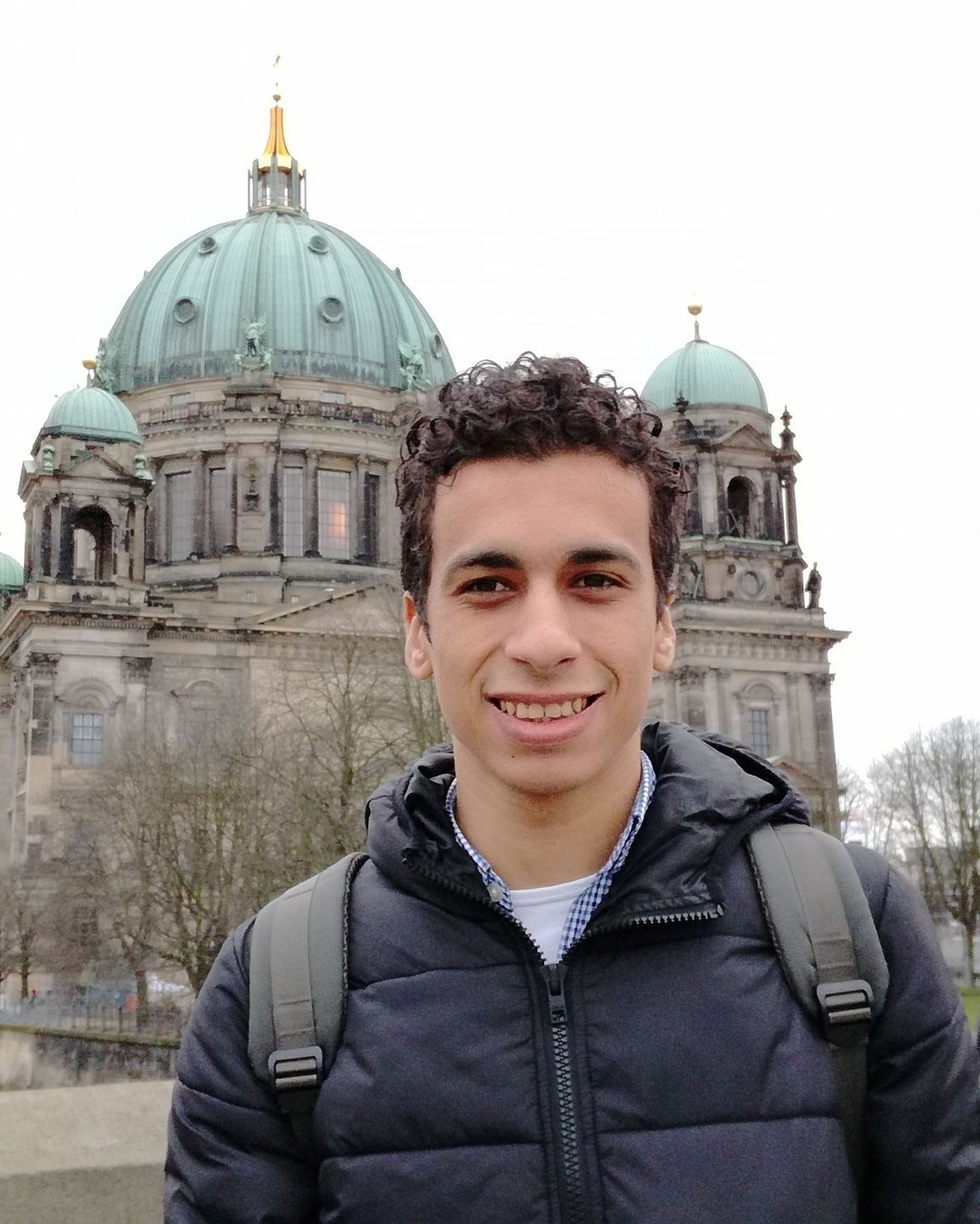}}]{Osama A. Hanna} is a Ph.D. candidate in the Electrical and Computer Engineering Department at the University of California, Los Angeles (UCLA). He received his BS and MS degrees in electrical engineering from the Faculty of Engineering Cairo University and Nile University in Egypt in 2014 and 2018 respectively. He received the Award of Excellence from Cairo University in 2014. He received the Masters Fellowship and a graduate Research Assistantship from Nile University for the years 2014-2018. He received the Electrical and Computer Engineering Department Fellowship from UCLA for the year 2018/2019. His research interests are machine learning, information theory and algorithms.
\end{IEEEbiography}

\begin{IEEEbiography}[{\includegraphics[width=1in]{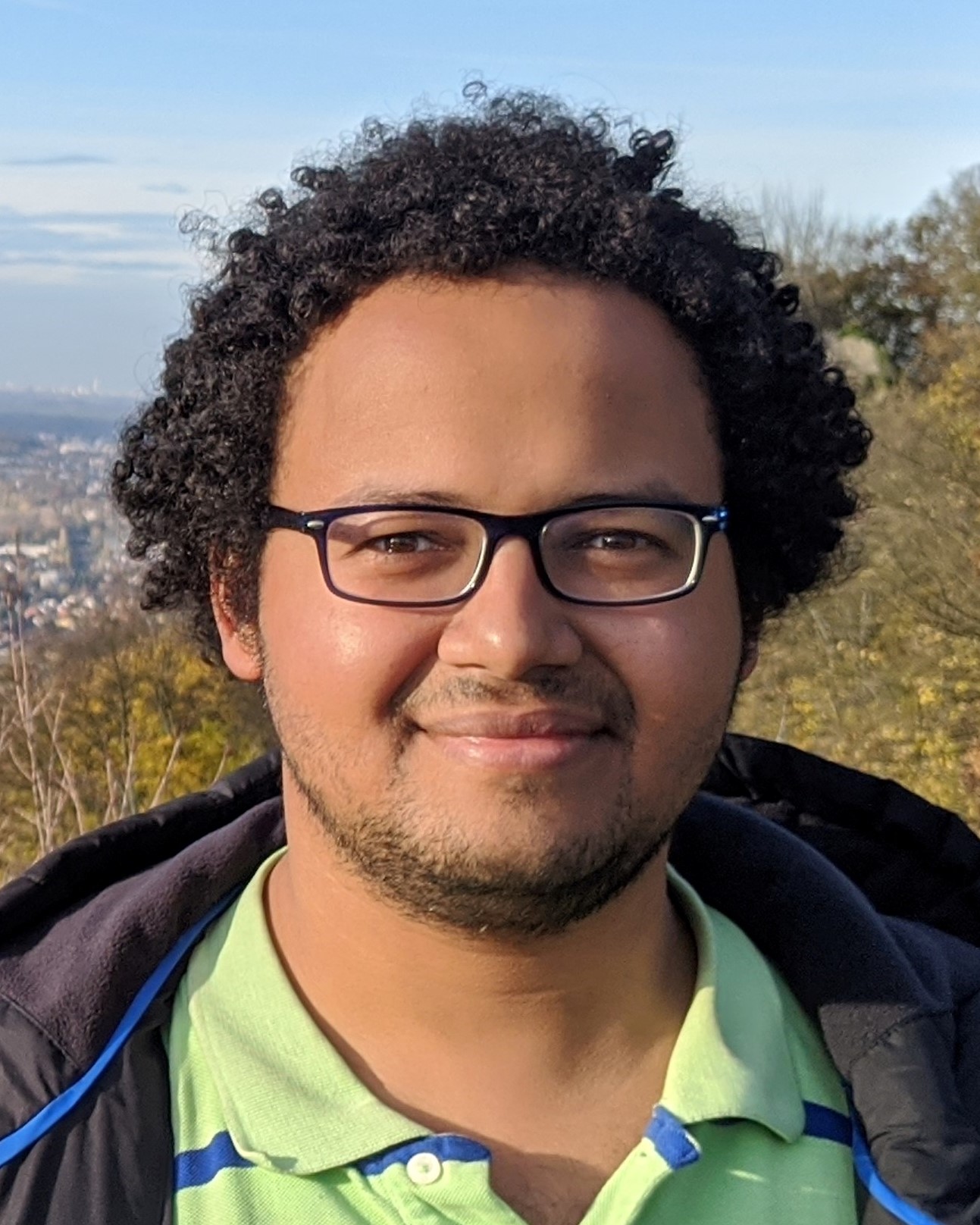}}]{Yahya H. Ezzeldin}
is currently a Postdoctoral Research Associate in the Electrical and Computer Engineering department at USC.
 He received his B.S. and M.S. degrees in Electronics and Communications Engineering from Alexandria University in 2011 and 2014, respectively. He received his Ph.D. degree in Electrical and Computer Engineering from UCLA in 2020. His current research interests include information theory, federated learning and distributed optimization. He worked as a machine learning platform engineer with Intel Corporation in the summer of 2018. He received the 2020-2021 Distinguished Ph.D. Dissertation Award in Signals and Systems from the Electrical and Computer Engineering Department at UCLA.
 He is also the recipient of the UCLA University Fellowship in 2014, the Henry Samueli Fellowship in 2016 and the Dissertation Year Fellowship at UCLA in 2019.
\end{IEEEbiography}

\begin{IEEEbiography}[{\includegraphics[width=0.9in]{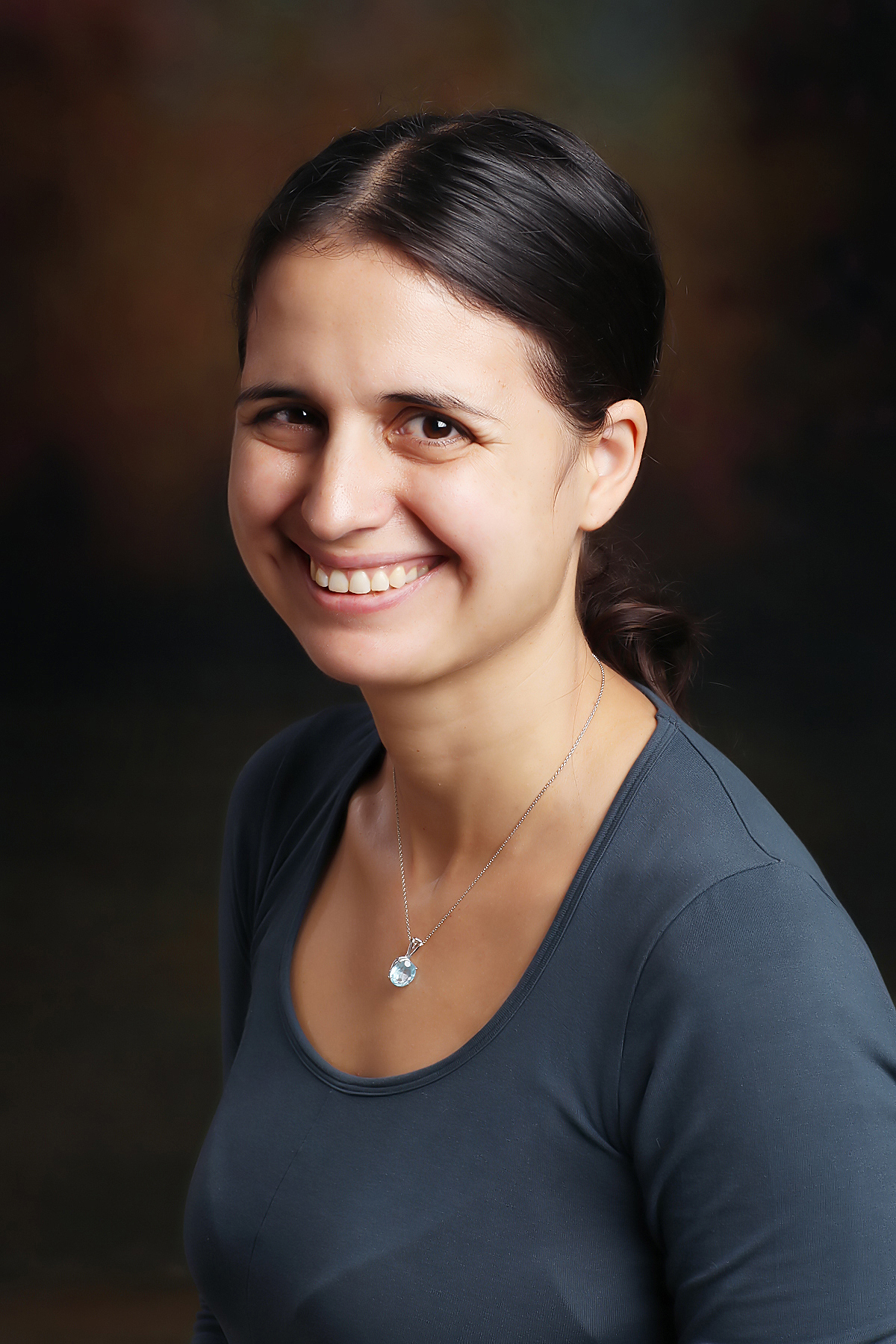}}]{Christina Fragouli} 
is a  Professor in the Electrical and Computer Engineering Department at UCLA. She received the B.S. degree in Electrical Engineering from the National Technical University of Athens, Athens, Greece, and the M.Sc. and Ph.D. degrees in Electrical Engineering from the University of California, Los Angeles. She has worked at the Information Sciences Center, AT\&T Labs, Florham Park New Jersey, and the National University of Athens. She also visited Bell Laboratories, Murray Hill, NJ, and DIMACS, Rutgers University. Between 2006--2015 she was  an Assistant and Associate Professor in the School of Computer and Communication Sciences, EPFL, Switzerland. She is an IEEE fellow, has served as  an Information Theory Society Distinguished Lecturer, and as an Associate Editor for IEEE Communications Letters,  for Elsevier Journal on Computer Communication, for IEEE Transactions on Communications, for IEEE Transactions on Information Theory, and for IEEE Transactions on Mobile Communications. She has also served in several IEEE committees, and received awards for her work.  Her research interests are in network information flow, network security and privacy, wireless networks and bioinformatics.
\end{IEEEbiography}

\begin{IEEEbiography}[{\includegraphics[width=1in]{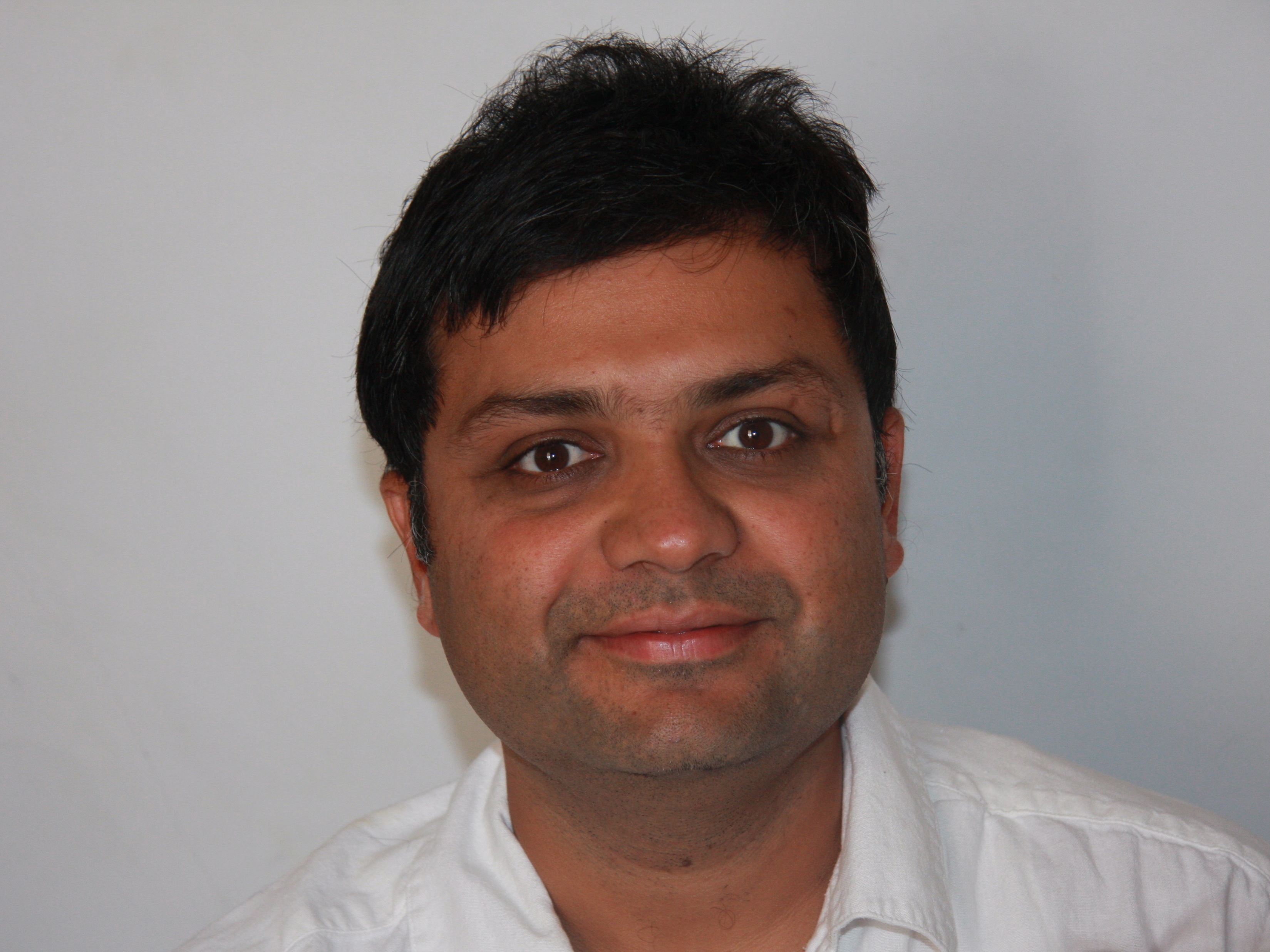}}]{Suhas Diggavi}
is currently a Professor of Electrical and Computer
Engineering at UCLA. His undergraduate education is from IIT, Delhi and his PhD is from Stanford University. He has worked as a principal member research staff at AT\&T Shannon Laboratories and directed the Laboratory for Information and Communication Systems (LICOS) at EPFL. At UCLA, he directs the Information Theory and Systems
Laboratory.

His research interests include information theory and its applications to several areas including machine learning, security \& privacy, wireless networks, data compression, cyber-physical systems, bio-informatics and neuroscience; more information can be found at http://licos.ee.ucla.edu.

He has received several recognitions for his research including the 2013 IEEE Information Theory Society \& Communications Society Joint Paper Award, the 2013 ACM International Symposium on Mobile Ad Hoc Networking and Computing (MobiHoc) best paper award, the 2006 IEEE Donald Fink prize paper award among others. He was selected as a Guggenheim fellow in 2021.  He also received the 2019 Google Faculty Research Award and 2020 Amazon faculty research award. He served as a
IEEE Distinguished Lecturer and also currently serves on board of governors for the IEEE Information theory society. He is a Fellow of the IEEE.

He has been an associate editor for IEEE Transactions on Information Theory, ACM/IEEE Transactions on Networking, IEEE Communication Letters, a guest editor for IEEE Selected Topics in Signal Processing and in the program committees of several IEEE conferences. He has also helped organize IEEE and ACM conferences including serving as the Technical Program Co-Chair for 2012 IEEE Information Theory Workshop (ITW), the Technical Program Co-Chair for the 2015 IEEE International Symposium on Information Theory (ISIT) and General co-chair for Mobihoc 2018. He has 8 issued patents.
\end{IEEEbiography}

\end{document}